\documentclass[sn-mathphys-num]{sn-jnl} % For LaTeX2e

\usepackage{multirow}
\usepackage{wrapfig}

\usepackage{hyperref}
\usepackage{url}
\usepackage{amssymb}
\usepackage{tikz}
\usetikzlibrary{decorations.pathreplacing,angles,quotes,arrows.meta}
\usepackage{mathtools}
\usepackage{svg}
\usepackage{multirow}
\usepackage{enumerate}
\newcommand{\norm}[1]{|| #1 ||}
\newcommand{\abso}[1]{| #1 |}
\newcommand{\defeq}{\vcentcolon=}
\newcommand{\eqdef}{=\vcentcolon}

\newcommand{\ts}{\textsuperscript}

\newcommand{\E}{\mathbb{E}}

\newcommand{\R}{\mathbb{R}}
\newcommand{\Z}{\mathbb{Z}}

\newcommand{\dd}{\mathrm{d}}

\newcommand{\calA}{\mathcal{A}}

\newcommand{\calE}{\mathcal{E}}

\newcommand{\calJ}{\mathcal{J}}
\newcommand{\calH}{\mathcal{H}}

\newcommand{\calL}{\mathcal{L}}
\newcommand{\calM}{\mathcal{M}}
\newcommand{\calN}{\mathcal{N}}
\newcommand{\calD}{\mathcal{D}}
\newcommand{\calF}{\mathcal{F}}

\newcommand{\raug}{R^{\mathrm{aug}}}

\newcommand{\calU}{\mathcal{U}}

\newcommand{\sse}{\subseteq}

\DeclareMathOperator{\SO}{SO}

\DeclareMathOperator{\id}{id}

\DeclareMathOperator{\Homm}{Hom}

\newcommand{\sprod}[1]{\langle #1 \rangle}

\newcommand{\rhorot}{\rho^{\mathrm{rot}}}
\newcommand{\rhoind}{\rho^{\mathrm{ch}}}
\newcommand{\calLsym}{\calL^{\mathrm{sym}}}
\newcommand{\calLas}{\calL^{\mathrm{as}}}
\newcommand\disteq{\stackrel{\mathclap{\normalfont\mbox{d}}}{=}}

\usepackage{amsthm}

\newtheorem{theorem}{Theorem}[section]
\newtheorem{lemma}[theorem]{Lemma}
\newtheorem{proposition}[theorem]{Proposition}

\theoremstyle{remark}
\newtheorem{remark}{Remark}
\newtheorem{example}{Example}[section]
\theoremstyle{plain}
\newtheorem{assumption}{Assumption}
\newtheorem{definition}{Definition}

\usepackage{booktabs}\let\cline\cmidrule

\abstract{
Recent work has shown that group equivariance emerges in ensembles of neural networks as the result of full augmentation in the limit of infinitely wide neural networks (neural tangent kernel limit). In this paper, we extend this result significantly. We provide a proof that this emergence does not depend on the neural tangent kernel limit at all. We also consider stochastic settings, and furthermore general architectures. For the latter, we provide a simple sufficient condition on the relation between the architecture and the action of the group for our results to hold. We validate our findings through simple numeric experiments.}

\title{Ensembles provably learn equivariance through data augmentation}

\author*{Oskar Nordenfors}
\author{Axel Flinth}
\affil{Department of Mathematics and Mathematical Statistics \\
Umeå University}
\affil{Linneaus väg 49, 901 87 Umeå}
\email{oskar.nordenfors, axel.flinth @umu.se}

\begin{document}
\maketitle

\section{Introduction}
Consider a learning task with an inherent symmetry. As an illustrative example, we can think of classifying images of, say, apples and pears. The classification should, of course, not change when the image is rotated, which means that the learning task is \emph{invariant} to rotations. If we segment the image, the apple segmentation mask should instead rotate with the image -- such tasks are called \emph{equivariant}. An a priori known symmetry like this is a strong inductive bias, which should be possible to use to improve performance of the model.

There are essentially two 'meta-approaches' to do the latter. The first is to, through data augmentation, ensure that the data respect the symmetry. In our example, one would add rotated versions of the training images to the dataset. This is versatile, simple, and effective. However, there are of course no guarantees that the model after training exactly obeys the symmetry, in particular on out-of-distribution data. To solve this problem, one can use the second meta-approach, namely, to use a model that inherently respects the symmetry. During the last half decade, an entire framework concerning so-called equivariant neural networks has emerged under the name of \emph{Geometric Deep Learning} (GDL) \cite{bronstein2021geometric}.

The relations between these two approaches are, from a theoretical point of view, still unclear: there is yet to emerge a definitive answer to the question of if and when the augmentation strategy is guaranteed to yield an equivariant model after training. Recently, an interesting discovery in this regard was made in \cite{gerken2024emergentequivariancedeepensembles}. In this paper, it was shown that \emph{ensembles} of \emph{infinitely wide} neural networks (in the neural tangent kernel limit \cite{jacot2018neural}) are equivariant \emph{in mean}, even though individual members of them are not. See Figure \ref{fig:ensemble} for an illustration. Although this asymptotic result is already interesting, it has some drawbacks. The proofs rely heavily on the simplified dynamics of the NTKs, and therefore do not work at all for network of finite width. Furthermore, the proof only works for finite groups, and is only formulated for the somewhat unrealistic gradient flow training. The purpose of this article is to show a similar result for more realistic scenarios.

\subsection{Literature review}

% Add approximate equivariance

% Add https://openreview.net/pdf?id=L86glqNCUj - gäller under några tekniska resultat, det viktiga är att vi gör det direkt  OCH vi behandlar mer allmänna nätverk, och ser hur geometrin påverkar.
\label{sec:litrev}
%To give a comprehensive overview of Geometric Deep Learning goes well beyond the scope of this paper. We instead refer to the textbook \cite{bronstein2021geometric}.

The question of the effect of symmetries in the data on the training of neural networks has often been understood as a comparison question: Is it better to train a non-equivariant architecture on augmented (that is, artificially symmetric) data, or to use a manifestly equivariant architecture as, e.g. \cite{cohen2018generalGCNN,maron2018invariant,kondorGeneralizationEquivarianceConvolution2018,weiler2019general,finzi2021practical,fuchs2020se}? Empirical comparisons between the two approaches as part of experimental evaluations of manifestly equivariant architectures have been made more or less since the invention of equivariant networks -- more systematic investigations include \cite{gandikotaTrainingArchitectureHow2021,gerken22a}.

In \cite{chen2020group}, a group theoretical framework for data augmentation was developed, which has inspired many treatises of the question, including ours. Linear (i.e. kernel) models are studied in \cite{elesedy2021provably,mei2021learning,dao2019kernel}. \cite{lyle2019analysis,lyle2020benefits} study so-called \emph{feature-averaged} networks. So-called \emph{linear neural networks} -- i.e., neural networks with a linear activation function -- are studied in \cite{chen2023implicit} and \cite{duan2025understandinglearninginvariancedeep}. In these settings, equivalence between augmenting the data and restricting the networks can be proven. The same has not been established for bona fide neural networks with non-linear activation functions -- they are studied in \cite{nordenfors2024optimizationdynamicsequivariantaugmented}, but there, only local guarantees are derived. Similar results appear in \cite{simsek2021geometry} where the invariance of certain symmetric subspaces of permutations of the weights under gradient flow was demonstrated.

The work most closely related to ours is \cite{gerken2024emergentequivariancedeepensembles}. There, it was realized that equivariance emerges from augmentation in ensembles, and not in individual networks.  They provide a formal explanation of this only in the so-called \emph{neural tangent kernel (NTK) limit} \cite{jacot2018neural}. Importantly, the optimization dynamics of neural networks in the NTK-limit turn linear \cite{lee2019wide}, so that that setting is considerably simpler than the one we consider here. We should also mention the concurrent \cite{maass2024symmetries}, which concerns the behavior of machine learning models trained on augmented data in the \emph{mean-field limit}, which include ensembles. This work is also indirect in that it technically concerns so-called Wasserstein gradient flows of certain functionals, which are limits of the ensemble dynamics. In neither work are architectures other than fully connected ones treated.

Note that this work is \emph{not} about \emph{actively} adapting the training of a priori non-equivariant models to make them (approximately equivariant), such as in e.g. \cite{pertigkiozoglou2024improving,Ouderaa2024Learning}. Instead, this work is about showing that equivariance emerges \emph{without regularization} in ensembles.

One important line of work, which we will rely on for the presentation of our results, is the study of equivariant training algorithms from \cite{li2021convolutional} and \cite{brugiapaglia2022invariance}. In an idealized setting, this, for example, refers to gradient flow on symmetric data. This flow will then be an equivariant flow.

Equivariant flows have been studied before in the context of generative models \cite{kohler2020equivariant,katsman2021equivariant,garcia2021n}. In this paper, we use results from \cite{kohler2020equivariant} in a different context, namely, to study the flow of the \emph{parameters of the network }under (stochastic) gradient flow (descent).

\subsection{Contribution}

The main message of this paper is that the phenomenon of neural network ensembles being equivariant in mean is \emph{much more general} than \cite{gerken2024emergentequivariancedeepensembles} suggests. In fact, the emerging equivariance has very little to do with the tangent kernel limit. What instead matter is the equivariance of the \emph{training algorithm} (which applies to gradient descent methods), and symmetric initialization of the weights. We are here heavily inspired by the works \cite{li2021convolutional} and \cite{brugiapaglia2022invariance}, but we apply them in a more general setting than before. For example, instead of making an assumption of symmetric data as in \cite{brugiapaglia2022invariance}, we also let the data symmetry be a property of the training algorithm.

\paragraph{More involved architectures} Where previous works have been confined to MLPs, i.e. fully connected neural networks, we derive results in a framework (that the authors developed themselves in \cite{nordenfors2024optimizationdynamicsequivariantaugmented}) that is capable to handle most modern architectures (including transformers). The idea here is to model a neural network architecture as a fully connected architecture with weights confined to an affine subspace $\calL$ of the parameter space $\calH$. We call $\calL$ the \emph{architecture space}. Our results reveal that the \emph{geometry} of $\calL$ is key -- only when $\calL$ is \emph{invariant to an action} of the symmetry group are we able to show that ensembles become equivariant. Although we not prove theoretically that the latter condition is necessary, we do numerical experiments that show that when it is not satisfied, the ensembles become less equivariant.

\paragraph{Random augmentations} Another contribution is that we treat the case of mini-batch stochastic gradient descent with randomly drawn data augmentations. This is in contrast to e.g. \cite{gerken2024emergentequivariancedeepensembles}, where the data usually is assumed to contain the entire orbits under the group action of every element of the data set. Circumventing the latter allows us to treat infinite and finite groups in a unified manner.

\paragraph{Finite ensembles} The main results of the paper only apply to the true mean of the ensemble, which is only realized in a theoretical infinite-ensemble-member limit. In practice, one of course need to resort to sample approximations of this mean. We therefore derive basic results on the sample complexity needed to achieve approximate equivariance. This contribution is minor and not the focus of the work -- the bounds are simple, but easy to interpret: We show that for a finite group, $O(\log(\vert G\vert))$ ensemble members are enough to realize the equivariance approximately for a given element $x$. For a connected Lie group, we show that $O(\mathrm{dim}(G))$ sample suffice.

\paragraph{Limitations} We make the same global assumptions as in \cite{nordenfors2024optimizationdynamicsequivariantaugmented}; The symmetry group needs to be compact (excluding e.g. the group of rigid motions $\mathrm{SE}(d)$), the augmentations are done according to the Haar measure on the group (meaning that they are uniform), and also that all non-linearities are equivariant with respect to group actions on the intermediate spaces of the network. The latter requirement is however satisfied in many natural architecture.

\paragraph{Paper outline} In Section \ref{sec:prel}, we define our setting, and in particular introduce equivariant training algorithms. We also show the abstract result Lemma \label{lem:equiensemble}: An ensemble of neural networks trained with an equivariant training algorithm becomes equivariant. In Section \ref{sec:calL}, we present the architecture space formalism, and prove some important lemmas. In \ref{sec:main}, we then prove the main results: Gradient flow with fully augmented data, and SGD with random augmentations, yield equivariant ensemble as long as the architecture space of the neural network is compatible with the group symmetry. We also show our concentration results for finite ensembles. In \ref{sec:exp}, we validate our findings with some simple numerical experiments.

\begin{figure}[!h]
    \centering
    %\begin{minipage}[c]{.4\textwidth}
    \includegraphics[width=.9\linewidth]{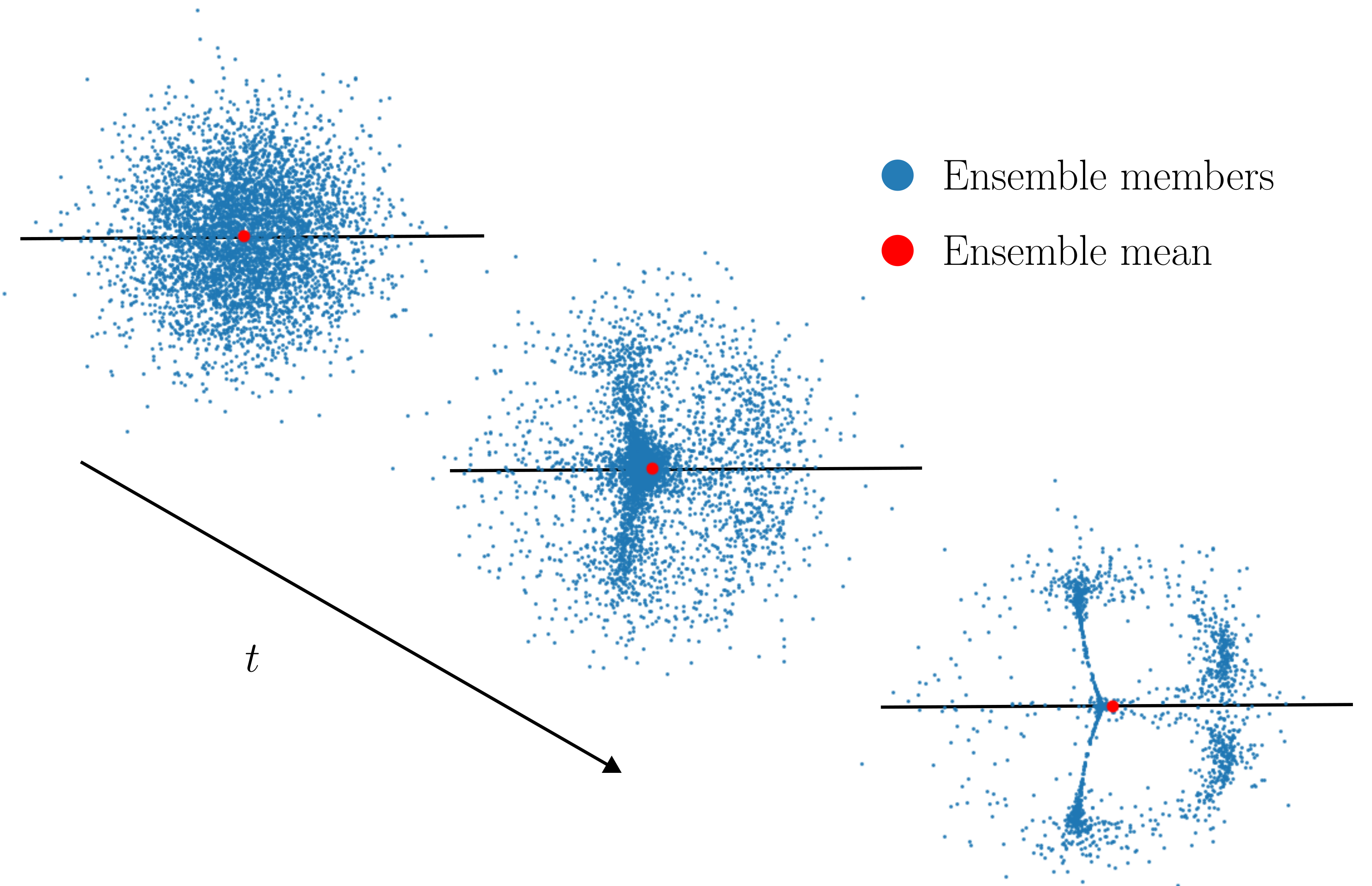}    
    %\end{minipage}
        %\begin{minipage}[c]{.55\textwidth}
    \caption{A graphical illustration of our results. The symmetry group here is $C_2$,  acting on the parameter space through reflection in the $x$-axis, so that the $x$-axis correspond to equivariant models. Snapshots of the parameters of ensemble members as they are trained on symmetric data are shown. At all times, most individual ensemble members do not lie on the line of symmetric models. However, their distribution always is symmetric about the $x$-axis -- and therefore, \emph{their mean} always corresponds to an equivariant model. \label{fig:ensemble} }
    %\end{minipage}
    
\end{figure}

\section{Preliminaries}\label{sec:prel}
Let $X$ and $Y$ be finite-dimensional vector spaces. We are concerned with neural networks $\Phi_A:X\to Y$, that are defined as function of the form 
\begin{align}\label{eq:raug}
    x_0\in X, \quad x_{i+1} = \sigma_i(A_i(x_i)), \, \forall i\in\{0,1,\ldots, N-1\}, \, \Phi_A(x_0)=x_N\in Y.
\end{align}
%By and large, we adopt the framework from \cite{nordenfors2024optimizationdynamicsequivariantaugmented} -- which itself is just a particular formalization of the general geometric deep learning framework \cite{bronstein2021geometric}. Let $(x,y)\in X\times Y$, where $X$ and $Y$ are (finite-dimensional) vector spaces. We are concerned with neural networks $\Phi_A:X\to Y$, that are defined as function of the form 
%\begin{align}\label{eq:raug}
%    x_0\in X, \quad x_{i+1} = \sigma_i(A_i(x_i)), \, \forall i\in\{0,1,\ldots, N-1\}, \, \Phi_A(x_0)=x_N\in Y.
%\end{align}
Here, $A_i$ are trainable linear maps, i.e. elements of $\Homm(X_i,X_{i+1})$, and $\sigma_i$ are fixed non-linearities.  We write $\calH=\bigoplus_{i=0}^{N-1}\Homm(X_i,X_{i+1})$ for the ambient space for the parameter $A=(A_i)_{i=0}^{N-1}$.

Let us proceed as in the GDL framework \cite{bronstein2021geometric} and define equivariance with respect to representations of a group. We assume that a group $G$ is acting on the spaces $X_i$ through some \emph{unitary representations} $\rho_i$ \cite{fulton2004representation} -- i.e. maps $\rho_i$ from the group $G$ to the set of unitary matrices that respect the group action: $\rho_i(gh)=\rho_i(g)\rho_i(h)$. Importantly, we can lift these representations to the space of weights.
\begin{example}[Trivial representation]
    The simplest possible representation on any space is the \emph{trivial} one, that maps every element $g\in G$ to the identity, $\rho^{\mathrm{triv}}(g)=\mathrm{id}$. In other words, $\rho^{\mathrm{triv}}(g)x = x$ for all $g\in G$. 
\end{example}

\begin{example}[Discrete rotations of images]\label{ex:rhorot}
    Let $G=C_4$ act on the space of $n\times n$ pixel images by representation $\rhorot$, which takes the input image and rotates it a multiple of 90 degrees. In other words, if $x:[n]^2\to \R$ is a grayscale image, then $\rhorot$ acts by
    $$
    \rhorot(g)x[i] = x[\pi_g^{-1}i],
    $$
    where $\pi_g$ is the permutation of the pixels corresponding to the rotation $g$. If we have a tuple of images, then $\rhorot$ acts by rotating each entry in the tuple in the same way.
\end{example}

\begin{definition}[Representation on $\Homm$]
    Given unitary representations $\rho_i$ on $X_i$, we define representations $\widehat{\rho}_i$ on $\Homm(X_i,X_{i+1})$ through
    \begin{align*}
        \widehat{\rho}_i(g)(A_i) = \rho_{i+1}(g) \circ A_i \circ \rho_i(g)^{-1}
    \end{align*}
    and a representation $\rho$ on $\calH$ as the direct sum of these. We will refer to these representations as the \emph{lifted representations}.
\end{definition}

\begin{example}
    If $G$ acts trivially on all spaces, it is not hard to see that the lifted representation $\rho$ is also trivial.
\end{example}

\begin{remark} \label{rem:free}
    While the representations on the input and output space are given by the symmetries of the problem and hence are fixed, the representations on the intermediate spaces are a priori free to choose. This point has subtle but important consequences for the meaning of our main results -- we will discuss them when we present it.
\end{remark}

We call a map $f: X\to Y$ $\rho$-\emph{equivariant} if $f\circ \rho_0(g) = \rho_N(g) \circ f$ for all $g$. The space of $\rho$-equivariant linear maps between $X_i$ and $X_{i+1}$ is denoted $\Homm_G(X_i,X_{i+1})$. We write $\calH_G=\bigoplus \Homm_G(X_i,X_{i+1})$, i.e., $\calH_G=\{A\in\calH\vcentcolon \rho(g)A=A\}$.

We will note two assumptions we will make throughout the paper. These assumptions also appear in \cite{nordenfors2024optimizationdynamicsequivariantaugmented}.
\begin{assumption}\label{ass:sigma}
    The non-linearities $\sigma_i$ are all $\rho$-equivariant.
\end{assumption}
Note that without this assumption the strategy of making a neural network equivariant by restricting the linear layers to be equivariant fails.

\begin{assumption}\label{ass:compact}
    The group $G$ is compact.
\end{assumption}

This assumption is needed to ensure that the Haar measure of the group \cite{krantz2008geometric} is (normalizeable to) a probability measure. It should also be noted that when $G$ is compact, it is no true restriction to assume that the representation is unitary -- one can always ensure this by redefining the inner products on $\calH$. The assumption is satisfied by all finite groups, but also groups like $\SO(n)$ and $\mathrm{O}(n)$ (albeit not the group of all rigid motions $\mathrm{SE}(n)$.)

\subsection{Training algorithms}
\label{train}
We will assume that we are given training data $(x,y)$ distributed according to some distribution $\calD$ on $X\times Y$. To produce a model from the training data we apply a training algorithm. We will use the definition of training algorithms and iterative training algorithms from \cite{li2021convolutional}.

\begin{definition}[Training algorithm]
    A \emph{training algorithm} $\calA$ is a mapping from a sample $ D=\{(x_i,y_i)\}_{i=1}^n$ from $\calD$ to a distribution $p$ on the space of neural networks $\calM=\{\Phi_A:X\to Y\}$. For a deterministic training algorithm, the output is instead simply a neural network.
\end{definition}

Most training algorithms used in practice are iterative, in the sense that they are defined by repeatedly applying some fixed  update rule $U$, that may have random components.  In other words, they are Markov chains dependent on the network architecture and data:
\begin{align*}
    \mathbb{P}(U(A^t,\Phi_{\cdot},D)\in S \, \vert \, A_t) = \int_{M} u_{\Phi,D}(A^t, a) \dd a,
\end{align*}
where the transition function $u_{\Phi,D}:\calL \times \calL$ is dependent on the data and the neural network architecture.
\begin{definition}[Iterative training algorithm]
    Given the following data:
    \begin{itemize}
        \item A sample $D$ from $\calD$
        \item An initial weight distribution $p_0$;
        \item A number of iterations $T$;
        \item A neural network architecture $\Phi_\cdot:\calL\to\calM$.
        \item An update rule $U$.
    \end{itemize}
    we define an \emph{iterative training algorithm} $\calA$ by the procedure:
    \begin{enumerate}
        \item Initialize $A^0\sim p_0$.
        \item For $t\in \{0,1,\ldots, T\}$, calculate recursively
        \begin{equation*}
            A^{t+1}=U(A^t,\Phi_\cdot,D)
        \end{equation*}
        \item Return $\Phi_{A^T}$.
    \end{enumerate}
\end{definition}
\begin{remark}\label{rem:iteratedupdate}
    Denoting by $\calU^t(D)$ the map applying the first $t$ iterations of the update rule, we have that if $A^T\disteq \calU^T(D)$, then $\Phi_{A^T} \disteq \calA (D)$.
\end{remark}
To achieve an equivariant model, it is not enough to merely consider training algorithms; we must consider \emph{equivariant} training algorithms. For a training sample $D=\{(x_i,y_i)\}_{i=1}^n$, let, for every $g\in G$, $gD\defeq \{(\rho_0(g)x_i,\rho_N(g)y_i)\}_{i=1}^n$ denote the transformed training sample.

\begin{definition}[Equivariant training algorithm]\label{def:eqvalgo}
    We say that a training algorithm $\calA$ is \emph{equivariant} if
    \begin{equation*}
        \rho_N(g)^{-1}\circ\calA(gD)\circ \rho_0(g) \disteq \calA(D), \quad \text{for all }g\in G, \,D\in (X\times Y)^n,
    \end{equation*}
\end{definition}

It is possible to show that iterative training algorithms are equivariant if the update is equivariant. To show this, we need the following Lemma from \cite{nordenfors2024optimizationdynamicsequivariantaugmented}.
\begin{lemma}[Equivariance to joint transformation]\label{lem:jointequi}
    Let $\Phi_A$ be a neural network, then, under Assumption (1) that all non-linearities are equivariant, we have that,
    \begin{align}
        \Phi_A(\rho_0(g)x)=\rho_N(g)\Phi_{\rho(g^{-1})A}(x),
    \end{align}
    for every $g\in G$ and $x\in X$.
\end{lemma}
\begin{proof}
    The proof is by induction on the number of layers $N$. To this end, denote by $\Phi^N_A$ the network given by the first $N$ layers of $\Phi_A$. For $N=1$ we have
    \begin{align*}
        \rho_1(g)\Phi^1_{\rho(g^{-1})A}(x) &=\rho_1(g)\sigma_0(\rho_1(g^{-1})A_0\rho_0(g^{-1})^{-1} x)= \sigma_0(\rho_1(g)\rho_1(g^{-1})A_0\rho_0(g^{-1})^{-1}x)\\
        &=\sigma_0(\id A_0\rho_0(g) x)=\Phi^1_A(\rho_0(g)x), \quad \forall g\in G,
    \end{align*}
    where the first equality is by definition of $\rho(g)$, the second equality is by $\rho$-equivariance of $\sigma_0$, and the third equality follows since $\rho_i(g)$ is a representation, $\forall i \in \{0,1\}$, $\forall g\in G$.

    Now, assume that $\Phi^K_A(\rho_0(g) x)=\rho_K(g)\Phi^K_{\rho(g^{-1})A}(x)$, $\forall g \in G$. We want to show that $\Phi^{K+1}_A(\rho_0(g) x)=\rho_{K+1}(g)\Phi^{K+1}_{\rho(g^{-1})A}(x)$, $\forall g \in G$. We have
    \begin{align*}
        \rho_{K+1}(g)\Phi^{K+1}_{\rho(g^{-1})A}(x)&= \rho_{K+1}(g)\sigma_{K}(\widehat{\rho}_K(g^{-1})A_{K}\Phi^K_{\rho(g^{-1})A}(x))\\
        &=\rho_{K+1}(g)\sigma_{K}(\rho_{K+1}(g^{-1})A_{K}\rho_{K}(g^{-1})^{-1}\Phi^K_{\rho(g^{-1})A}(x))\\
        &=\sigma_{K}(\rho_{K+1}(g)\rho_{K+1}(g^{-1})A_{K}\rho_{K}(g^{-1})^{-1}\Phi^K_{\rho(g^{-1})A}(x))\\
        &=\sigma_{K}(\id A_{K}\rho_{K}(g)\Phi^K_{\rho(g^{-1})A}(x))\\
        &=\sigma_{K}( A_{K}\Phi^K_{A}(\rho_0(g)x))=\Phi^{K+1}_{A}(\rho_0(g)x),\quad \forall g\in G,
    \end{align*}
    where the first equality is by the definition of $\Phi_A$, the second equality is by the definition of $\widehat{\rho}_K(g)$, the third equality follows by $\rho$-equivariance of $\sigma_K$, the fourth equality follows since $\rho_i(g)$ is a representation, $\forall i \in \{K,K+1\}$, $\forall g\in G$, the fifth equality follows by the inductive assumption, and the final equality is just the definition of $\Phi^{K+1}$.

    By induction $\Phi_A(\rho_0(g) x)=\rho_N(g)\Phi_{\rho(g^{-1})A}(x)$, $\forall g \in G$.
\end{proof}

This lemma says that if  the non-linearities of our network are equivariant, then the entire network is equivariant with respect to the group representation $\rho_\calJ$ on $\calH\oplus X$ defined by $\Phi(\rho_\calJ(g)(A,x))=\Phi_{\rho(g)A}(\rho_0(g)x)$ and the representation $\rho_N$ on the output space. This can be seen as a relaxed kind of equivariance. Note that if $\rho(g)A=A$, $\Phi_A$ becomes equivariant -- this is, somewhat streamlined, the core principle behind GDL. Now, note that this holds 'in distribution' as well: If $A\disteq \rho(g)A$ for every $g\in G$, then
$$
\Phi_A(\rho_0(g) x)=\rho_N(g)\Phi_{\rho(g^{-1})A}(x)\disteq \rho_N(g)\Phi_{A}(x) , \quad \forall g \in G.
$$
 We can use this idea to prove the stochastic version of Theorem C.1 from \cite{li2021convolutional}.
\begin{lemma}[Equivariance of update step implies equivariance of training algorithm]\label{lem:equiupdate} 
    Let $\calA$ be an iterative training algorithm. If the update rule satisfies
    \begin{align}\label{eq:equiupdate}
        \rho(g)U(A^t,\Phi_{\cdot}, D) \disteq U(\rho(g)A^t,\Phi_{\cdot}, gD), 
    \end{align}
    and the initial weight distribution satisfies $p_0 = p_0\circ \rho(g)^{-1}$, for every $g\in G$, then $\calA$ is an equivariant training algorithm.
\end{lemma}

\begin{proof}
    The proof is by induction. Let $\calU^t$ denote the first $t$ iterations of the iterative update rule $U$, so that if $A^T\disteq \calU^T(D)$, then $\Phi_{A^T}\disteq \calA(D)$. The base case is just the invariance assumption: $\rho(g)A^0\disteq A^0$. 
    
    For the inductive step, assume that $\rho(g)A^t\disteq \calU^t(gD)$, then we have
    \begin{align*}
        \rho(g)A^{t+1}&=\rho(g)U(A^t,\Phi_{\cdot}, D) \\&\disteq U(\rho(g)A^t,\Phi_{\cdot}, gD) \disteq \calU^{t+1}(gD),
    \end{align*}
  where the last equality is by inductive assumption. By induction, $\Phi_{\rho(g)A^T}\disteq \calA (gD)$. It follows by Lemma \ref{lem:jointequi} that
    \begin{align*}
        \calA(D)\circ \rho_0(g)^{-1}=\Phi_{A^T}\circ \rho_0(g)^{-1} = \rho_N(g)^{-1}\circ\Phi_{\rho(g)A^T} \disteq  \rho_N(g)^{-1}\circ\calA (gD),
    \end{align*}
    which is what was to be shown.
\end{proof}

\subsection{Ensembles}
%The goal of this paper is to show a similar statement to  \emph{ensembles} trained on augmented data will be equivariant for all times $t$ if they are properly initialized. 

Ensembling refers to the practice of training a number of independent base models, and then averaging them.  In our setting, we may define them as follows:
%We define ensembles as in \cite{gerken2024emergentequivariancedeepensembles}: For a distribution $p$ on $\calH$, we consider the distribution of parameters $A^t$ implied by initializing the parameters of a network according to $p$, and then training the network for time $t$. The corresponding \emph{ensemble model} at time $t$ is then defined as 
\begin{align}\label{eq:ensemble}
    \overline{\Phi}^t(x)=\E[\Phi_{A^t}](x), \quad \Phi_{A^t} \disteq \calA(D).
\end{align}
When the training algorithm is gradient flow, this definition is the same as in \cite{gerken2024emergentequivariancedeepensembles}. 

One should note that this is an idealization. In practice, one needs to resort to a sample mean of the distribution $\calA(D)$ via initializing $M$ neural networks (according to the initialization defined by $\calA$), training them independently, and then averaging them after $t$ steps. This is the strategy we will follow in the experimental section. The bulk of our claims will be about the the true mean, but will get back to the effects of finite sampling in Section \ref{sec:finite}. %We will refer to such collections of $M$ networks as an \emph{ensemble}, and each individual network as an \emph{ensemble member}.

If an ensemble is trained with an equivariant algorithm on symmetric data, its mean will be equivariant. More specifically, we have the following.
\begin{theorem}[Equivariant training leads to equivariant ensembles]\label{thm:equiensemble}
    Assume that an ensemble model is trained with an equivariant training algorithm $\calA$ on a dataset $D$ that satisfies \begin{align}\label{eq:symdata}\calA(D) \disteq \calA(gD),\end{align}for every $g\in G$. Then, the ensemble model is equivariant.
\end{theorem}
\begin{proof}
    Let $\calA$ be an equivariant training algorithm and let $\Phi_{A^t}\disteq\calA(\{x_i,y_i\}_{i=1}^n)$. Then, by Equation \ref{eq:symdata}, we have,
\begin{align*}
    \Phi_{A^t} \circ \rho_0(g) \disteq \rho_N(g) \circ \Phi_{A^t}.
\end{align*}
Taking the expectation of both sides yields
\begin{align*}
    (\E[\Phi_{A^t}]\circ\rho_0(g)) (x) = (\rho_N(g)\circ\E[\Phi_{A^t}])(x),
\end{align*}
for every $g\in G$ and $x\in X$. This is the same as,
\begin{align*}
    \overline{\Phi}_t(\rho_0(g)x) = \rho_N(g)\overline{\Phi}_t(x).
\end{align*}
\end{proof}

\begin{remark} \label{rem:symmetric_data}
     If the dataset is symmetric, that is if $$D\disteq gD,$$ then the condition $$\calA(D) \disteq \calA(gD)$$ holds trivially. This observation was made in \cite{brugiapaglia2022invariance}.  
\end{remark}

\begin{lemma}(Data symmetric update yields data symmetric algorithm)\label{lem:datasymupd}
    Let $\calA$ be an iterative training algorithm with update rule $U$.  If $U$ satisfies \begin{align}\label{eq:datasymupdate}U(A^t,\Phi_\cdot,D)\disteq U(A^t,\Phi_\cdot,gD),\end{align} for every $g\in G$ and every $D$, then $\calA$ satisfies $\calA(D) \disteq \calA(gD)$ for every $g\in G$ and every $D$.
\end{lemma}
\begin{proof}
    We proceed by induction. Let $\calA$ be an iterative training algorithm with an update rule $U$ satisfying equation \ref{eq:datasymupdate}. Letting $\calU^t$ denote the first $t$ iterations of the iterative update rule $U$, we have that
    \begin{align*}
        \calU^1(D) = U(A^0, \Phi_{\cdot}, D) \disteq U(A^0, \Phi_{\cdot}, gD) =\calU^1(gD),
    \end{align*}
    by Equation \ref{eq:datasymupdate}. For the inductive step, if $\calU^t(D) \disteq \calU^t(gD)$, then
    \begin{align*}
        \calU^{t+1}(D) \disteq U(\calU^t(D),\Phi_{\cdot}, D)\disteq U(\calU^t(gD),\Phi_{\cdot}, D) \disteq U(\calU^t(gD),\Phi_{\cdot}, gD)\disteq \calU^{t+1}(gD),
    \end{align*}
    where the first and last equalities are by definition, the second equality by the inductive assumption, and the third equality is by Equation \ref{eq:datasymupdate}. The statement now follows by induction.
\end{proof}
Thus, we see that symmetric data is not a necessity for Equation \ref{eq:symdata} to hold -- it suffices that the update rule does not change if we transform the data. In Section \ref{sec:main}, we will show that mini-batch stochastic gradient descent with data batches that are augmented 'on the fly'.
 %we consider a distribution $p$ on $\calH$, such that $A\sim p$, and define an ensemble model for time $t$ by \begin{align*}
  %  \overline{\Phi}_t(x)=\E_{p}[\Phi_{\calF_t A}(x)].
%\end{align*}
%Note that we are here considering the true expectation over $p$, and not a sample mean from $p$. In the sequel we will also write $A\sim B$ to mean that $A$ and $B$ are equal in distribution.

% abstrahera resonemanget: Både 'stokastisk gradient' och gradient flow är ekvivarianta 

% Inledning: Det är i princip välkänt att ekvivariant träning + invariant initialisering leder till invariant fördelning hela tiden. [referenser] låt oss ge ett bevis här också at the convenience of the reader [Prop]

% I detta papper kommer vi undersöka under vilka villkor vi kan applicera detta för neurala nätverk.'

% Globala antagaden

\section{The architecture space formalism} \label{sec:calL}
 The model defined in Section \ref{sec:prel} is really only a faithful model for fully-connected feedforward neural networks without bias. As was realized in \cite{nordenfors2024optimizationdynamicsequivariantaugmented}, we may conveniently include other architectures in the framework (biases, CNN:s, RNN:s, transformers,~\dots) by confining the tuple of linear maps to an \emph{affine subspace} $\calL \sse \calH$ which we will refer to as the \emph{architecture space}. Let recall and motivate this definition by giving some examples of architecture spaces, including one that we will use as a running example in the following.
\begin{example}[Trivial architecture space]
    In a fully connected network without bias, all layers are linear, and the weights are not constrained at all. Hence, MLP:s without bias have architecture space $\calL=\calH$. We refer to this as the \emph{trivial architecture space}.
\end{example}

\begin{example}[Architectures with bias]
    Biases can be included by the standard 'homogenization' trick -- add a 'dummy feature' to all intermediate spaces $\widetilde{X}_i= X_i \oplus \mathbb{R}$, then restrict the maps between $\widetilde{X}_i$ and $\widetilde{X}_{i+1}$ to the set
    \begin{align*}
        \mathcal{B}_i = \bigg\{ H_{A,b} = \begin{bmatrix}
            A & b \\ 0  & 1
        \end{bmatrix} \, \big\vert \, A \in \Homm(X_i,X_{i+1}) , b\in X_{i+1} \bigg\}.
    \end{align*}
    The linear maps in $\mathcal{B}_i$ then emulate the action of affine maps: $H_{A,b}(x,1) = (Ax+b,1)$. The architecture space of the net as a whole is then a direct sum: $\calL= \oplus_{i\in [L]} \mathcal{B}_i$. Note that it is always possible to construct an architecture space 'layer-wise' in this manner.
\end{example}

\begin{figure}[!h]
\centering
        \begin{tikzpicture}[scale=.39]
            \draw[step=1cm] (0,0) grid (3,3);
            \draw[fill=gray!50] (1,1) rectangle (2,2);
            \draw[fill=gray!50] (0,1) rectangle (1,2);
            \draw[fill=gray!50] (2,1) rectangle (3,2);
            \draw[fill=gray!50] (1,0) rectangle (2,1);
            \draw[fill=gray!50] (1,2) rectangle (2,3);
            \draw[step=1cm] (6,0) grid (9,3);
            \draw[fill=gray!50] (6,2) rectangle (7,3);
            \draw[fill=gray!50] (6,1) rectangle (7,2);
            \draw[fill=gray!50] (8,1) rectangle (9,2);
            \draw[fill=gray!50] (7,0) rectangle (8,1);
            \draw[fill=gray!50] (7,2) rectangle (8,3);
        \end{tikzpicture}
        \caption{Symmetric (left) and asymmetric (right) filter. Indices in the support are grey.}
        \label{fig:supports}
\end{figure}
\begin{example}[Running example: Convolutional layers]    \label{ex:liftedconv}

    For some integers $d_i$, let $X_i = (\R^{n,n})^{d_i}$ denote spaces of tuples of images. We may define a \emph{convolutional operator} between two such spaces as operators $C$ of the form
    \begin{align*}
        (Cx)_k = \sum_{\ell \in [d_i]} \psi_{k\ell}*x_\ell,
    \end{align*}
    where $\psi_{k\ell}$, $k \in [d_{i+1}], \ell \in [d_i]$ are filters. Let us denote this operator by $\sprod{\psi}$. 
    The space of all operators of this form, where the filters supports are all contained in a fixed set $\Omega$, form a linear (and therefore affine) subspace $\calL$ of $\calH$. The corresponding architecture is a CNN.

    Common choices for $\Omega$ are $3\times 3$ or $5\times 5$ squares. Just as was done in \cite{nordenfors2024optimizationdynamicsequivariantaugmented}, we will additionally here consider the two non-canonical supports in Figure \ref{fig:supports}, since they illustrate important aspects of our results. We call the left support the \emph{symmetric} support, and the right the \emph{asymmetric} support, and denote the corresponding architecture spaces $\calLsym$ and $\calLas$.

\end{example}
As is carefully presented in \cite{nordenfors2024optimizationdynamicsequivariantaugmented}, many more architectures can be thought of as an MLP with weights confined in an architecture space $\calL$, including e.g. RNNs and transformers. Importantly, the non-linearities $\sigma$ often need to be chosen in a non-standard way (in for instance transformers, $\sigma$ is responsible for the attention mechanism). We refer to the appendix of said article for detailed arguments.

Let us now state an assumption we will make throughout the text.
\begin{assumption}\label{ass:nonempty}
    The space $\calE = \calL\cap\calH_G$ is non-empty.
\end{assumption}

This assumption ensures that there exists an admissible and $\rho$-equivariant architecture. It also has a technical consequence that we will need: it allows us to write the architecture space $\calL=A_\calL+\mathrm{T}\calL$, for an $A_\calL \in \calH_G\cap \calL$, and $\mathrm{T}\calL$  the tangent space of $\calL$. 

\subsection{The invariance criterion}
In \cite{nordenfors2024optimizationdynamicsequivariantaugmented}, neural network training with augmented data was compared with the 'layerwise equivarization strategy' from geometric deep learning. The latter, in our language, simply means confining the layers $A_i$ to be equivariant with respect to the intermediate representations, i.e. $\calH_G$. It turned out that the sets $\calL\cap \calH_G$ are invariant under augmented training dynamics \emph{under some additional geometric assumptions on $\calL$}. The story in this paper will be the same: In order for ensembling to yield equivariance, $\calL$ needs to be compatible with the lifted representation, in the following manner.  
\begin{definition} \label{defi:ginv}
    Let $\calL\subseteq \calH$. The space $\calL$ is $\rho$-invariant if $\rho(g)\calL \sse \calL$ for all $g\in G$.
\end{definition}
Intuitively, $\rho$-invariance of $\calL$ can be interpreted as the corresponding architecture not being \emph{inherently biased} against being $\rho$-equivariant.

The condition of $\rho$-invariance was again already used in \cite{nordenfors2024optimizationdynamicsequivariantaugmented}, to derive the so-called \emph{compatibility condition} (meaning that $\Pi_\calL$ commutes with the orthogonal projection $\Pi_G$ onto $\calH_G$) needed to establish the results of that paper. Many examples of $\rho$-invariant architecture spaces corresponding to commonly used architectures and reasonable group representations were given.  Examples included fully architectures connected ones with and without bias, RNNs, residual connections, etc. Let us here discuss a few important examples.

\begin{example}[Convolutions]\label{ex:convcond}
    Consider the $\calL$ corresponding to a convolutional architecture, and let $C_4$ act through $\rhorot$ (see Example \ref{ex:rhorot}) on all intermediate spaces. It is not hard to show that the the corresponding lifted representation acts by rotating each filter, $\rho(g) \sprod{\psi} = \sprod{\rhorot(g)\psi}$. (For a formal proof, see for example \cite{nordenfors2024optimizationdynamicsequivariantaugmented}). Consequently, a $\sprod{\psi} \in \calLsym$ will be mapped by $\rho$ onto a tuple of other symmetric filters, so that $\calLsym$ is $\rho$-invariant. By the same argument, $\calLas$ is not.
     As we previously remarked (Remark \ref{rem:free}), there is however no need to choose the $\rhorot$-representations in the intermediate layers. However, it turns out that representations on the intermediate spaces making $\calLas$ $\rho$-invariant do not exist. The somewhat involved proof for this is given in Appendix \ref{app:skew}.
\end{example}
\begin{example}[Trivial architecture spaces and representations]\label{ex:trivarch} The trivial architecture space is invariant independently of $G$ and $\rho$, since we of course always have $\rho(g) \calH \sse \calH$, by the virtue of $\rho$ being a representation. On the other hand, if the lifted representation is trivial, any architecture space $\calL$ will be $\rho$-invariant, since  $\rho^{\mathrm{triv}}(g)\calL =\mathrm{id}(\calL) = \calL$ for all $g\in G$ in that case. 
\end{example}

\begin{example}[Pre- and post MLP:s] \label{ex:prepost}
    In fact, the observations in the previous examples one can be used to argue that any architecture in a certain sense can be modified into an invariant one. Let us assume that an architecture with architecture space $\calL$ for parametrizing functions $X \to Y$ is given. The representations of the group on $X$ and $Y$, $\rho_0$ and $\rho_N$, are fixed, but as we remarked in \ref{rem:free}, the ones on the intermediate spaces are not. Now let us modify our architecture as follows: Instead of directly inputting the data in $X$ into the network, we pre-compose it with a fully connected MLP $X\to X$. In the same manner, we do not use the output of the network directly, but post-compose it with a fully connected MLP $Y\to Y$. The architecture space of the resulting architecture gets the form $\calL^{\mathrm{extended}} =\calL^{\mathrm{pre}} \oplus \calL \oplus \calL^{\mathrm{post}}$. 

    Now, we choose the intermediate representations $\rho_i$ on all intermediate spaces to be trivial. It is then not hard to see that the lifted representation $\rho$ on $\calL^{\mathrm{extended}}$ decomposes as $\rho^{\mathrm{pre}} \oplus \rho^{\mathrm{triv}} \oplus \rho^{\mathrm{post}}$, where $\rho^{\mathrm{pre}}$ and $\rho^{\mathrm{post}}$ are some possibly non-trivial representations. Now, since $\calL^{\mathrm{pre}}$ and $\calL^{\mathrm{post}}$ are trivial architecture spaces, they  will be invariant under $\rho^{\mathrm{pre}}$ and $\rho^{\mathrm{post}}$, respectively. On the other hand, $\calL$ will of course be invariant under $\rho^{\mathrm{triv}}$. Hence, $\calL^{\mathrm{extended}}$ will \emph{always} be $\rho$-invariant. We record this in Proposition \ref{prop:prepost} below. 
    
    As for our global assumptions, note that all non-linearities in the resulting structures map between spaces on which $G$ acts trivially, so that \ref{ass:sigma} is satisfied. As for the assumption \ref{ass:nonempty} of $\calH_G$ being non-empty, note the zero map is always equivariant with respect to $\rho^{\mathrm{pre}}$ and $\rho^{\mathrm{triv}}$, and  $\rho^{\mathrm{triv}}$ and $\rho^{\mathrm{post}}$, respectively. It is also clear that the zero map is contained in the non-restricted $\calL^{\mathrm{pre}}$- and $\calL^{\mathrm{post}}$-spaces, so that $\calH_G\cap \calL$ is still non-empty after composing with an MLP.

    Using pre-and post MLP:s like this may be detrimental for performance -- all spatial structure of the data is destroyed, for instance. The construction is still interesting from a theoretical point of view, as it works for any group and representation, even a priori unknown ones.
\end{example}

\begin{proposition}[Composing with MLPs yields invariance] \label{prop:prepost}
    Any architecture can be turned into an invariant one by pre-and post composing it with MLP:s, and (~re~-~)~choosing all intermediate representations to be trivial.
\end{proposition}

Before moving on, let us prove one crucial lemma for what follows: $\rho$-invariance of $\calL$ is equivalent to $\rho$-equivariance of the projection $\Pi_\calL$. Note that this a well-known result from linear algebra. We include a proof for the convenience of the reader.
\begin{lemma}[Orthogonal projection onto invariant subspace is equivariant]\label{lem:eqviproj}
   $\calL$ is $\rho$-invariant  if and only if $\Pi_{\calL}$ is $\rho$-equivariant.
\end{lemma}
\begin{proof}
    Let us begin by proving the \emph{only if} part. Assume that $\rho(g)\calL \subseteq \calL$, $\forall g\in G$. Clearly, if $\exists B\in \mathrm{T}\calL$, $\exists g \in G$, such that $\rho(g)B\notin \mathrm{T}\calL$, then $A=A_{\calL}+B \in \calL$, but $\rho(g)A = \rho(g)A_{\calL} + \rho(g)B=A_{\calL} + \rho(g)B\notin \calL$. By contraposition, we must have $\rho(g)\mathrm{T}\calL\subseteq \mathrm{T}\calL$, $\forall g \in G$, since we have assumed that $\rho(g)\calL \subseteq \calL$, $\forall g\in G$. Now, suppose that $x\in \mathrm{T}\calL$ and $y\in \mathrm{T}\calL^{\perp}$. Then we have
    \begin{align*}
        \sprod{\rho(g)y,x}=\sprod{y,\rho(g)^*x} = \sprod{y,\rho(g)^{-1}x} =\sprod{y,\rho(g^{-1})x}=0,\quad \forall g\in G,
    \end{align*}
    where the second equality is by unitarity, the third equality is by definition of representations, and the final equality follows from $\rho(g)\mathrm{T}\calL\subseteq \mathrm{T}\calL$, $\forall g \in G$. Thus, it holds that $\rho(g)\mathrm{T}\calL^{\perp}\subseteq \mathrm{T}\calL^{\perp}$, $\forall g \in G$. We can decompose any $z\in \calH$ as $z=z_{\mathrm{T}\calL}+z_{\mathrm{T}\calL^{\perp}}$, where $z_{\mathrm{T}\calL}= \Pi_{\calL}z$ and $z_{\mathrm{T}\calL^{\perp}}=\Pi_{\calL}^{\perp}z$. It follows that
    \begin{align*}
        \Pi_{\calL} \rho(g) z = \Pi_{\calL} \rho(g)(z_{\mathrm{T}\calL}+z_{\mathrm{T}\calL^{\perp}})=\Pi_{\calL} \rho(g)z_{\mathrm{T}\calL}+\Pi_{\calL} \rho(g)z_{\mathrm{T}\calL^{\perp}}=\rho(g)z_{\mathrm{T}\calL}
        =\rho(g)\Pi_{\calL}z,
    \end{align*}
    $\forall g\in G$, where the second equality is by linearity, the third equality follows by $\rho(g)\mathrm{T}\calL\subseteq \mathrm{T}\calL$ and $\rho(g)\mathrm{T}\calL^{\perp}\subseteq \mathrm{T}\calL^{\perp}$, $\forall g \in G$, and the last equality is by definition of $z_{\mathrm{T}\calL}$. This concludes the \emph{only if} part of the proof.

    Let us now prove the \emph{if} part of the proof. Assume that $\Pi_\calL$ is $\rho$-equivariant. Then we have for any $B\in \mathrm{T}\calL$ that
    \begin{align*}
        \rho(g)B=\rho(g)\Pi_\calL B = \Pi_\calL \rho(g) B\in\mathrm{T}\calL,\quad \forall g\in G,
    \end{align*}
    since $\Pi_\calL$ is a $\rho$-equivariant projection onto $\mathrm{T}\calL$. Thus, we have for any $A\in \calL$ that
    \begin{align*}
        \rho(g)A =\rho(g)(A_\calL + B) = \rho(g)A_\calL + \rho(g)B = A_\calL + \rho(g)B\in \calL, \quad \forall g \in G,
    \end{align*}
    since $\calL=A_\calL+\mathrm{T}\calL$ with $\rho(g)A_\calL= A_\calL$ and $\rho(g)$ is linear. 
\end{proof}

%Let us begin by commenting on one of the assumptions in Theorem \ref{thm:main}. Namely, the assumption that $\rho(g)\calL\subseteq \calL$. This assumption can be  The condition is fulfilled for the admissible layers of many common architectures, including attention layers, fully connected with bias, and residual connections. One situation where the condition is not necessarily fulfilled is when $H$ is a subgroup of $G$ and $\calL=\calH_H$, but then replacing $g\in G$ for $h\in H$ and $\rho$-equivariance for $H$-equivariance, the theorem will still hold. Furthermore, as was shown in \cite{nordenfors2024optimizationdynamicsequivariantaugmented}, the condition is violated whenever $\Pi_{\calL}\Pi_{\calH_G}\neq \Pi_{\calH_G}\Pi_{\calL}$, where $\Pi_{\calL}$ and $\Pi_{\calH_G}$ are the orthogonal projections onto $\mathrm{T}\calL$ and $\calH_G$ respectively.

% Definiera invarianskriteriet

% Exempel diverse.
% Feature extraction!!! [MLP först+ sist]

%Convolutions

% motexempel

\section{Equivariance of augmented neural network training}\label{sec:main}

This section will be devoted to proving our main theorem: training a neural network with a $\rho$-invariant architecture space $\calL$ with augmentation will yield equivariant ensembles. We will consider two algorithms: gradient flow with 'full augmentation' (i.e., training using all transformed versions of each data points with infinitesimal learning rate) and stochastic gradient descent with random augmentations.
The strategy is clear: each of the algorithms, we will show that they are equivariant in the sense of Definition \ref{def:eqvalgo} and fulfill Equation \ref{eq:symdata}, and then apply Theorem \ref{thm:equiensemble}.

Before going into the details of each algorithm, let us give the general setting: We consider a fixed neural network architecture $\Psi_A$, with weights $A$ confined to an architecture space $\calL$. Given a set $D$ of training data from $X \times Y$ and a loss function $\ell:Y \times Y \to \R$, we define a \emph{nominal risk function}
\begin{align*}
    R_D(A) = \frac{1}{\vert{D}\vert} \sum_{(x,y) \in D}\ell(\Psi_A(x),y)).
\end{align*}
In order to prove our results, we will need to make one more global assumption. 
\begin{assumption}\label{ass:invariantloss}
    The loss function $\ell$ is $\rho$-invariant, that is $\ell(\rho_N(g)y,\rho_N(g)y')=\ell(y,y').$
\end{assumption}
This assumption can be interpreted as saying that the loss should not be biased towards non-equivariant functions. This assumption was also used in \cite{nordenfors2024optimizationdynamicsequivariantaugmented}, and it was there argued that it is often satisfied. As an example, if $G$ acts trivially on $Y$ (i.e., we are striving for an invariant model), $\ell$ is trivially $\rho$-invariant.

We are going to train the neural network by minimizing modified versions of $R$ over $\calL$ using modified versions of gradient flow/descent algorithms. To do so practically, we use a parametrization of $\calL$  $c \mapsto A_\mathcal{L} + Lc$, where $A_\mathcal{L}$ is a 'basepoint' in $\calL\cap \calH_G$, and $L: \R^p \to \mathrm{T}\mathcal{L}$ an isometric embedding, meaning that $L^*L= \mathrm{id}$, and then optimize the function
\begin{align*}
    \mathrm{R}(c) = R(A_\mathcal{L}+Lc).
\end{align*}
Notice that in the examples we have discussed, such an $L$ is easy to describe -- for example in the convolutional case, $L$ takes a vector of pixel values and inserts them to the pixels of the filter.

Importantly, calculating gradients of $\mathrm{R}(c)$ and then using them to update $A$ through the parametrization is the same as calculating gradients of $R$ and then projecting them to $\mathrm{T}\calL$, since
\begin{align*}
    L\nabla\mathrm{R}(c)= LL^*\nabla R(A_\mathcal{L} + Lc) = \Pi_{\mathcal{L}}\nabla R(A),
\end{align*}
where $LL^*=\Pi_\calL$ follows from the fact that $L$ is an isometric embedding.

To satisfy the conditions of Lemma \ref{lem:equiupdate}, we need to have $\rho$-invariantly initialized weights. Note that a Gaussian initialization of the coefficients $c$ will, as soon as $\calL$ is $\rho$-invariant, yield a $\rho$-invariant distributions of the parameters $A$:

\begin{lemma}[Invariant distribution of coefficients implies invariant distribution of weights]\label{lem:gaussinit}
    Let $L: \R^p \to \calL$ be a parametrization of $\calL$, $\calL = A_\calL + L\R^p$, with $A_\calL \in \calH_G$. Then, if $c$ is a standard Gaussian vector (with i.i.d. entries),  and $\calL$ is $\rho$-invariant, $A = A_\calL + Lc$ will have a $\rho$-invariant distribution.
\end{lemma}
\begin{proof}
    Because $c$ is standard Gaussian, $A$ will also have a Gaussian distribution, with mean $A_\calL$ and covariance matrix $LL^*=\Pi_\calL$, where the latter follows from the fact that $L$ is unitary. Now, by the same argument, $\rho(g)A= \rho(g)A_\calL + \rho(g)Lc$ will be Gaussian with mean $\rho(g)A_\calL=A_\calL$ -- remember that $A_\calL\in \calH_G$ -- and covariance $\rho(g)L(\rho(g)L)^*$.  We may now argue
    \begin{align*}
        \rho(g)L(\rho(g)L)^* = \rho(g)LL^*\rho(g)^* = \rho(g)\Pi_\calL \rho(g)^* = \Pi_\calL\rho(g)\rho(g)^* = \Pi_\calL.
    \end{align*}
    The penultimate step follows from the $\rho$-invariance of $\calL$, and the final one from the unitarity of $\rho(g)$. This shows that $A$ and $\rho(g)A$ both are Gaussians with mean $A_\calL$ and covariance $\Pi_\calL$, and we are done.
    \end{proof}

    \begin{remark}
        Note that there is of course no necessity to use a Gaussian initialization in general. For instance, if the intermediate representations are all trivial (which we know we can always use in the MLP setting, see Example \ref{ex:trivarch}), $A\disteq \rho(g)A$ is actually a vacuous condition for all but the first layer. Note that Li, Zhang and Arora in \cite{li2021convolutional} remark that they only need to assume that the first layer is invariantly initialized to prove their results -- since they are using MLPs, this is in perfect resonance with our argument here.
    \end{remark}

\subsection{Gradient flow}
Let us first consider the strategy of applying gradient flow to a fully augmented risk. Gradient flow hereby refers to letting the parameters $A$ flow according to the ODE
\begin{align*}
    \dot{A} = -\Pi_\calL\nabla R_D(A)
\end{align*}
for some fixed time $t$, from some initialization $A_0$. Equivalently, we can phrase this as applying the \emph{flow map} $\calF^t_D$ of $R_D$ once. That is, apply the update map
\begin{align*}
    U^{\mathrm{GF}}(A^0, \Phi_{\cdot},D) = \calF_D^t(A^0)
\end{align*}
once to a randomly initialized weight $A^0$.

By Lemma \ref{lem:equiupdate}, this yields an equivariant training algorithm if the initialization is $\rho$-invariant and the update is  $U^{\mathrm{GF}}$ is equivariant.

\begin{lemma}[Equivariance of $U^{GF}$] \label{lem:gradientflow}
    In addition to Assumptions \ref{ass:sigma}, \ref{ass:compact}, \ref{ass:nonempty}, and \ref{ass:invariantloss}, assume that  $\calL$ is $\rho$-invariant. Assume further that $R_D$ continuously differentiable in $A$. Then, $U^{\mathrm{GF}}$ is equivariant.
\end{lemma}
\begin{proof}
The smoothness assumption shows that the flow map is well-defined (by Picard-Lindelöf).

    Let us first show that under our assumption, we have $R_{gD}(\rho(g)A) = R_{D}(A)$.
    \begin{align*}
        R_{gD}(\rho(g)A) &=  \frac{1}{\vert{D}\vert} \sum_{(x,y) \in D}\ell(\Psi_{\rho(g)A}(\rho_0(g)x),\rho_N(g)y)) \\
        &\stackrel{\text{Lem.} \ref{lem:jointequi}}{=} \frac{1}{\vert{D}\vert} \sum_{(x,y) \in D}\ell(\rho_N(g)\Psi_A(x),\rho_N(g)y)) \\
        &\stackrel{\ell \text{ inv.}}= \frac{1}{\vert{D}\vert} \sum_{(x,y) \in D}\ell(\Psi_A(x),y) = R_D(A).
    \end{align*}
    Differentiating the equality $R_{gD}(\rho(g)A) = R_{D}(A)$ yields $\nabla R_D(A) = \rho(g)^*\nabla R_{gD}(\rho(g)A)$, which together with the unitarity of $\rho$ yields
    \begin{align} \label{eq:gradient_transformation}
        \nabla R_{gD} (\rho(g)A) = \rho(g)\nabla_D R(A).
    \end{align}

    Equipped with \eqref{eq:gradient_transformation}, the equivariance of $U^{\mathrm{GF}}$ now follows by some standard ODE theory arguments: By definition, what is to be proven is $\calF_{gD}(\rho(g)A^0) = \rho(g)\calF_{D}(A^0)$. That is:
    \begin{align*}
        \text{ If $\gamma$ solves} \begin{cases} \dot{\gamma}(s) &= - \Pi_\calL \nabla_DR (\gamma(s)) \\ \gamma(0) &= A^0\end{cases}, \text{ $\gamma^g :=\rho(g)\gamma$ solves } \begin{cases} \dot{\gamma^g}(s) &= - \Pi_\calL \nabla_DR (\gamma^g(s)) \\ \gamma^g(0) &= \rho(g)A^0\end{cases}
    \end{align*}
    It is clear that $\gamma^g(0) = \rho(g)\gamma(0) = \rho(g)A^0$. As for the differential equation, we have
    \begin{align*}
        \dot{\gamma^g}(s)&\stackrel{\rho(g) \text{ linear}}{=}\rho(g)\dot{\gamma}(s) \stackrel{\gamma \text{ solves ODE}}{=}-\rho(g)\Pi_\calL \nabla_D R(\gamma(s)) \stackrel{\text{Lem. } \ref{lem:eqviproj}}{=}-\Pi_\calL \rho(g) \nabla_DR(\gamma(s)) \\
        &\stackrel{\text{\eqref{eq:gradient_transformation}}}{=} -\Pi_\calL \nabla R_{gD}(\rho(g)\gamma(s)) = -\Pi_\calL \nabla R_{gD}(\gamma^g(s)).
    \end{align*}
\end{proof}

One way to view this result is that it is essentially a consequence of the fact that an equivariant vector field yields an equivariant flow map, which is proved in \cite{kohler2020equivariant}.

\begin{theorem}
    Ensembles of i.i.d invariant initialized neural networks with a $\rho$-invariant architecture space trained with gradient flow on fully augmentated datasets are equivariant. 
\end{theorem}
\begin{proof}
    Combine Lemmata \ref{lem:equiupdate}, \ref{lem:gradientflow}, Theorem \ref{thm:equiensemble} and \ref{rem:symmetric_data}.
 \end{proof}

\begin{remark}
    As we have previously noted, full group augmentation is practically impossible when dealing with infinite groups. Theorem \ref{thm:main} is however still relevant for it. That is, we could approximate the full augmentation by augmenting with a finite subset $H$ of $G$. Then, the risk we should optimize is $
        R^{H}(A)=\frac{1}{\abso{H}}\sum_{h\in H}R(\rho(h)A).$
    %The gradient of this risk satisfies
    %\begin{align*}
    %\nabla R^{H}(\rho(g)A)=\frac{1}{\abso{H}}\sum_{h\in H}\rho(h)^*\nabla R(\rho(h)\rho(g)A) &= \frac{1}{\abso{H}}\sum_{h\in H}\rho(g)\rho(g)^*\rho(h)^*\nabla R(\rho(h)\rho(g)A)\\
  %  &= \rho(g)\frac{1}{\abso{H}}\sum_{h\in H}\rho(hg)^*\nabla R(\rho(hg)A), \,\,%\forall g\in G,
%\end{align*}
If $H$ is a subgroup of $G$, this directly corresponds to performing full augmentation with the smaller set $H$. For instance, one could approximate $\mathrm{SO}(3)$ by the subgroup of rotational symmetries of an icosahedron. Our results then immediately imply that the ensemble models will be equivariant with respect to that subgroup.% If $H$ approximates $G$ well enough, this should mean that the ensemble becomes approximately
%Thus, it follows that if $H$ is a subgroup of $G$, then $\nabla R^{H}$ is $H$-equivariant, which further yields that $\calF^t$ is $H$-equivariant if also $\rho(h)\calL\subseteq \calL$, $\forall h \in H$. Therefore, if $H$ is a subgroup of $G$, then $\overline{\Phi}_t$ is $H$-equivariant aFs long as $p$ is $\rho$-invariant and $\rho(h)\calL\subseteq \calL$, $\forall h \in H$. .
\end{remark}

%Corollary: Gradient flow with augmented data leads to equivariant ensembles 
\subsection{Stochastic gradient descent}
We now move on to \emph{mini-batch stochastic gradient descent (SGD) with random augmentation}. This is again an iterative algorithm, with steps defined as follows: At each step, we uniformly draw a batch $B_t \sse D$ of $s$ data-examples from the training set $D$ together with $s$ i.i.d Haar-distributed group elements $g_k$. We apply $(\rho_0\oplus \rho_N)(g_k)$ to $(x_k,y_k)$, for every $k\in \{1,\ldots,s\}$, then calculate the average gradient of the estimated risk over the batch. Let us denote
$$
\widehat{\raug_D} (A) \defeq \frac{1}{s}\sum_{k=1}^s\ell (\Phi_{A}(\rho_0(g_k)x_k),\rho_N(g_k)y_k)
$$
then, given some sequence of step-sizes $\{\gamma_t\}_{t=1}^{T}$, we can define the update rule
\begin{align}
    A^{t+1} = U^{SGD}(A^t,\Phi_\cdot, D) \defeq A^t - \gamma_{t+1}\Pi_\calL\nabla \widehat{\raug_D}(A^t).
\end{align}
We will assume that the draw of data points and group elements for each batch is independent of each other for every time step and of the weights at every previous time step.% SGD is an equivariant training algorithm.
\begin{remark}
    Independence of the draw of batches and the weights is a faithful model for SGD where the batches are drawn with replacement, which also includes the deterministic GD setting. The assumption of independence of the batches $B_t$ and weights $\{A^k\}_{k=0}^{t-1}$ is necessary for our proofs, but is an idealization of the way SGD really works. In practice, the batches in mini-batch SGD are \emph{dependent} since they are sampled without replacement.
\end{remark}
Note that the projected descent emerges in the same way as above by updating the parametrization coefficients $c^t$. Each realization of the process $A^t$ corresponds to one run of the SGD.
We want to show that mini-batch SGD with random augmentation yields an equivariant ensemble, when applied to invariantly initialized weights. Let us begin by proving the following very intuitive lemma about random functions: If two random functions, that almost surely are differentiable, have the property that their pointwise distributions are equal, the same is true for their gradients. For the proof, we will invoke the Cramér-Wold Theorem \cite{Cramr1936SomeTO}.
\begin{theorem}[Cramér-Wold]\label{thm:cramerwold}
    A probability distribution $p$ on $\R^d$ is uniquely determined by the set of distributions $p\circ\pi_v ^{-1}$ on $\R$, indexed by the unit vector $v$, where $\pi_v:\R^d\to \R$, $x\mapsto \sprod{x,v}$.
\end{theorem}
We now formulate and prove the lemma.
\begin{lemma}[Equidistributed functions have equidistributed gradients]\label{cor:cramergradients}
    Let $\varphi$ and $\psi$ be random functions $\R^d\to \R$ that almost surely are $C^1$. If for every $x\in\R^d$, $\varphi(x)\disteq \psi(x)$, then  also $\nabla \varphi(x) \disteq \nabla \psi(x)$, for all $x\in\R^d$.
\end{lemma}
\begin{proof}
    Assume that $\varphi(x)\disteq \psi(x)$, for all  $x\in\R^d$. Then
    \begin{align*}
        X_n \defeq \frac{\varphi(x+\frac{1}{n}v)-\varphi(x)}{\frac{1}{n}} \disteq \frac{\psi(x+\frac{1}{n}v)-\psi(x)}{\frac{1}{n}}\eqdef Y_n, \quad \forall x,v\in \R^d, \, n\in\Z_+.
    \end{align*}
    Now, for every $x\in\R^d$ and every $v\in S^{d-1}$, $X_n\to \sprod{\nabla \varphi(x),v}$ and $Y_n\to\sprod{\nabla \psi(x),v}$ a.s. as $n\to\infty$. Thus, the distributions of the limit random variables (i.e., the distributions of the directional derivative in direction $v$) also agree. So that
    \begin{align*}
        \sprod{\nabla \varphi(x),v}\disteq\sprod{\nabla \psi(x),v},\quad \forall x\in \R^d, \, \forall v\in S^{d-1}.
    \end{align*}
    Since this holds for every $x\in\R^d$ and every $v\in S^{d-1}$, we can apply the Cramér-Wold Theorem to conclude that
    \begin{align*}
        \nabla \varphi(x) \disteq \nabla \psi(x),\quad \forall x\in\R^d.
    \end{align*}
\end{proof}

From the above it follows that an invariant function has an equivariant gradient.

\begin{lemma}[Invariant function has equivariant gradient]\label{lem:equidistr}
    If $f$ is a random function $\calH \to \R$ that almost surely is $C^1$ and satisfies $f(A) \disteq (f\circ \rho(g))(A)$, for every $g \in G$ and every $ A\in \calH$, then $\nabla f$ satisfies $\nabla f(\rho(g)A) \disteq \rho(g)\nabla f(A)$, for every $g \in G$ and every $A\in \calH$.
\end{lemma}
\begin{proof}
    We have that
    \begin{align*}
        \nabla f (A) \disteq \nabla (f\circ \rho(g)) (A) 
        = \rho(g)^* \nabla f (\rho(g)A), \quad \forall g \in G,\quad \forall A\in \calH,
    \end{align*}
    where the first equality in distribution follows from Lemma \ref{cor:cramergradients} with $\varphi=f$ and $\psi=f\circ \rho(g)$, and the second  is simply the chain rule. The result follows by applying $\rho(g)$ to both sides of the equality.
\end{proof}
To apply the above result to our risk, we need to demonstrate its invariance. Let us do this next.
\begin{lemma}[Invariance of $\widehat{\raug_D}$]\label{lem:invrisk}
    The risk $\widehat{\raug_D}$ satisfies
    \begin{enumerate}[(i)]
        \item $\widehat{\raug_{gD}}(A) \disteq \widehat{\raug_D}(A)$,
        \item $\widehat{\raug_D}(\rho(g)A)\disteq \widehat{\raug_D}(A)$,
    \end{enumerate}
    for every $g\in G$, $A\in \calH$, and $D$.
\end{lemma}
\begin{proof}
    We will begin by proving (i). We have for $h\in G$ arbitrary, due to $\rho_0$ and $\rho_N$ being distributions,
    \begin{align*}
        \widehat{\raug_{hD}}(A)&= \frac{1}{s}\sum_{k=1}^s \ell(\Phi_{A}(\rho_0(g_k)\rho_0(h)x_k),\rho_N(g_k)\rho_N(h)y_k) \\&=\frac{1}{s}\sum_{k=1}^s \ell(\Phi_{A}(\rho_0(g_kh)x_k),\rho_N(g_kh)y_k).
    \end{align*}
   Now, by the invariance property of the Haar measure, the tuple of group elements $g_1, \dots g_{s}$ are equidistributed with the tuple $g_1h, \dots g_sh$. Since the data points $(x_k,y_k)$ are independent of the draw of the $g_k$, we can conclude that  the above is equidistributed with 
    \begin{align*}
 \frac{1}{s}\sum_{k=1}^s \ell(\Phi_{A}(\rho_0(g_k)x_k),\rho_N(g_k)y_k)= \widehat{\raug_D}(A),
    \end{align*}
  which is what we wanted to show.
    Next, we will prove (ii). Applying Lemma \ref{lem:jointequi}, we get that for every $A\in \calH$ that
    \begin{align*}
        \widehat{\raug_D}(\rho(h)A)&= \frac{1}{s}\sum_{k=1}^s \ell(\Phi_{\rho(h)A}(\rho_0(g_k)x_k),\rho_N(g_k)y_k) \\
        &= \frac{1}{s}\sum_{k=1}^s \ell(\rho_N(h)\Phi_{A}(\rho_0(h)^{-1}\rho_0(g_k)x_k),\rho_N(g_k)y_k).
    \end{align*}
    Invoking the $\rho$-invariance of $\ell$ and then that $\rho$ is a representation, the above can be rewritten to
    \begin{align*}
        & \frac{1}{s}\sum_{k=1}^s \ell(\Phi_{A}(\rho_0(h^{-1}g_k)x_k),\rho_N(h^{-1}g_k)y_k)
    \end{align*}
    Now, by the invariance property of the Haar measure, the tuple of group elements $g_1, \dots g_{s}$ are equidistributed with the tuple $h^{-1}g_1, \dots h^{-1}g_{s}$. Since the data points $(x_k,y_k)$ are independent of the draw of the $g_k$, we can conclude that  the above is equidistributed with 
    \begin{align*}
 \frac{1}{s}\sum_{k=1}^s \ell(\Phi_{A}(\rho_0(g_k)x_k),\rho_N(g_k)y_k)= \widehat{\raug_D}(A),
    \end{align*}
  which is what we wanted to show.
\end{proof}
We now show that the gradient of $\widehat{\raug_D}$, in a distributional sense, is equivariant. 
%We will now show that $R^g$ satisfies $R^g(A) \sim (R^g\circ \rho(g))(A)$, for every $g\in G$ and every $A\in \calH$, and thus that the gradient of $R^g$ satisfies $\nabla R^g(\rho(g)A) \sim \rho(g)\nabla R^g(A)$, for every $g \in G$ and every $A\in \calH$. %the  %This is the counterpart of Lemma \ref{lem:eqvigrad} for SGD with random augmentation.
\begin{lemma}[Equivariance of $\nabla \widehat{\raug_D}$]\label{lem:equidistrgrad}
    The gradient of $\widehat{\raug_D}$ satisfies
    \begin{enumerate}[(i)]
        \item $\nabla \widehat{\raug_{gD}}(A) \disteq \nabla \widehat{\raug_{D}}(A)$
        \item $\nabla \widehat{\raug_D}(\rho(g)A) \disteq \rho(g)\nabla \widehat{\raug_D}(A)$,
    \end{enumerate}for every $g \in G$, $A\in \calH$, and $D$. 
\end{lemma}
\begin{proof}
    (i) follows by Lemma \ref{cor:cramergradients} and (ii) by Lemmata \ref{lem:equidistr} and \ref{lem:invrisk}.
\end{proof}

\begin{lemma}[Equivariance and data symmetry of $U^{\mathrm{SGD}}$]\label{lem:eqvsgd}
    Let $\calL$ be $\rho$-invariant, then the update rule $U^{\mathrm{SGD}}$ satisfies
    \begin{enumerate}[(i)]
        \item $U^{\mathrm{SGD}}(A,\Phi_\cdot, gD)\disteq U^{\mathrm{SGD}}(A,\Phi_\cdot, D)$,
        \item $\rho(g)U^{\mathrm{SGD}}(A,\Phi_\cdot, D)\disteq U^{\mathrm{SGD}}(\rho(g)A,\Phi_\cdot, gD)$,
    \end{enumerate}
    for every $g \in G$, and $D$.
\end{lemma}
\begin{proof}
    Let us first prove (i). By Lemma \ref{lem:equidistrgrad}(i)., we have that
    \begin{align*}
        U^{\mathrm{SGD}}(A,\Phi_\cdot, gD)\disteq A - \gamma\Pi_\calL\nabla \widehat{\raug_{gD}}(A) \disteq A - \gamma\Pi_\calL\nabla \widehat{\raug_{D}}(A)\disteq U^{\mathrm{SGD}}(A,\Phi_\cdot, D),
    \end{align*}
    which proves (i). For the second part, we will again use Lemma \ref{lem:equidistrgrad}, which yields
    \begin{align*}
        \rho(g)U^{\mathrm{SGD}}(A,\Phi_\cdot, D) &\disteq \rho(g)\big( A - \gamma\Pi_\calL\nabla \widehat{\raug_{D}}(A)\big) \disteq \rho(g)A - \rho(g)\gamma\Pi_\calL\nabla \widehat{\raug_{D}}(A)\\&\disteq \rho(g)A - \gamma\Pi_\calL\rho(g)\nabla \widehat{\raug_{D}}(A)\disteq \rho(g)A - \gamma\Pi_\calL\nabla \widehat{\raug_{D}}(\rho(g)A)\\&\disteq \rho(g)A - \gamma\Pi_\calL\nabla \widehat{\raug_{gD}}(\rho(g)A) \disteq U^{\mathrm{SGD}}(\rho(g)A,\Phi_\cdot, gD),
    \end{align*}
    where the third equality follows by $\rho$-invariance of $\calL$  and Lemma \ref{lem:eqviproj}, the fourth equality follows from Lemma \ref{lem:equidistrgrad}(ii), and the fifth equality follows from Lemma \ref{lem:equidistrgrad}(i). %, which proves (ii) for deterministic $A$. By the assumption of independence of the draw of batches and group elements and the weights, it follows that the above holds for non-deterministic $A$.
\end{proof}
With this, we are now equipped to prove the main result of the paper.
\begin{theorem}[Main theorem: Equivariance of ensembles trained by mini-batch SGD with random augmentations]\label{thm:main}
    Ensembles of i.i.d invariant initialized neural networks with a $\rho$-invariant architecture space trained with mini-batch SGD with random augmentations are equivariant.
\end{theorem}
\begin{proof}
    By Lemmata \ref{lem:eqvsgd}, \ref{lem:equiupdate}, and \ref{lem:datasymupd}, mini-batch SGD with random augmentations satisfies the assumptions of Theorem \ref{thm:equiensemble}, by which it follows that ensembles trained by mini-batch SGD with random augmentations are equivariant.
\end{proof}
%Corollary: SGDTräning ger invariata ensembler

\subsection{Finite ensembles}\label{sec:finite}

%Hoeffding-based bound, $\log(\abs{G})$ eller $\dim(G)$

   All of the results up until now are concerned with  the properties of $\overline{\Phi}_t$, i.e. the true mean of the ensemble over the random draw of the initialization and the randomness of the training algorithm. In practice, one of course needs to resort to approximate this mean via sampling: That is, draw $M$ independent initializations $A^0_k$, train each model separately to weights $A_k^t$, and then evaluate the sample mean
    \begin{align}
        \widehat{\Phi}^t_M(x) = \tfrac{1}{M}\sum_{k=1}^M \Phi_{A_k^t}(x). \label{eq:sample_ensemble}
    \end{align}
   Appealing to the law of large numbers, this sample mean of course get close to $\overline{\Phi}_t(x)$ as $M$ grows large, and the speed will essentially depend on the variance of the random variable $\Phi_{A_k^t}(x)$. To estimate the latter is in general hard -- in the NTK-limit it is easy, since a formula for the variance can be obtained. Gerken and Kessel used the latter to derive a simple sample bound in \cite{gerken2024emergentequivariancedeepensembles}. To obtain general bounds for finite-width neural networks is hard and well beyond the scope of this paper.

   Let us here only formalize the intuitive picture we gave above: Any bound on the variance of $\Phi_{A_k^t}$ (over different independent training runs) will give us a sample bound for achieving approximate equivariance. 
   \begin{proposition}[Approximate equivariance of finite ensembles]\label{prop:concentration}
       Let $Y=\mathbb{R}^d$ be equipped with the standard Euclidean norm. Assume that the distribution $\Phi_{A^t}(x)$ of the neural network outputs for an element $x\in X$ is
       \begin{enumerate}[(i)]
           \item almost surely bounded, $\norm{\Phi_{A^t}(x)}_2\leq \xi_{t}$
           \item has a bounded variance $\mathbb{E}(\norm{\Phi_{A^t}(x)-\overline{\Phi}^t(x)}_2^2) =V_{t}<\infty$.
       \end{enumerate}
        \underline{Case 1:} The orbit $O_x = \{\rho_0(g)x \, \vert \, g \in G\}$ is finite.      
       Then, for every $\varepsilon, \delta >0$ if 
       \begin{align*}
           M \geq C \cdot\max \bigg(\frac{V_{t}}{\varepsilon^2}, \frac{\xi_t}{\varepsilon}\bigg)\log\bigg(\frac{(d+1)\cdot\vert{O_x}\vert}{\delta}\bigg)
       \end{align*}
       the sampled ensemble $\widehat{\Phi}^t_N(x)$ will obey
       \begin{align*}
           \sup_{g\in G} \norm{\rho_N(g)\widehat{\Phi}^t_M(x) - \widehat{\Phi}^t_M(\rho_0(g)x)}_2 \leq \varepsilon
       \end{align*}
        with a probability higher than $1-\delta$. Here, $C>0$ is a constant independent of $ \Phi, x, G$.

\noindent
        \underline{Case 2:} The orbit $O_x$ is infinite. Then, under the additional assumption that  
        \begin{enumerate}
            \item[(iii)] $\Phi_{A^t}$ almost surely has a Lipschitz-constant smaller than $K$,
        \end{enumerate}
        we have for every $\varepsilon, \delta >0$ that if 
        \begin{align*}
           M \geq C \cdot\max \bigg(\frac{V_{t}}{\varepsilon^2}, \frac{\xi_t}{\varepsilon}\bigg)\log\bigg(\frac{(d+1)\cdot \calN(O_x, \tfrac{\varepsilon}{3\max(K,1)})}{\delta}\bigg)
       \end{align*}
       the sampled ensemble $\widehat{\Phi}^t_M(x)$ will obey
       \begin{align*}
           \sup_{g\in G} \norm{\rho_N(g)\widehat{\Phi}^t_M(x) - \widehat{\Phi}^t_M(\rho_0(g)x)}_2 \leq \varepsilon
       \end{align*}
        with a probability higher than $1-\delta$. Here, $C>0$ is a constant independent of $ \Phi, x, G$, and $\calN(O_x,\cdot)$ denotes the covering numbers for $O_x$. 
   \end{proposition}
       
    \begin{proof}
        Note that by the triangle inequality and Theorem \ref{thm:main} we have
        \begin{align}
            &\norm{\rho_N(g)\widehat{\Phi}^t_M(x) - \widehat{\Phi}^t_M(\rho_0(g)x)}_2 \nonumber\\
            &\qquad \qquad \leq \norm{\rho_N(g)\widehat{\Phi}^t_M(x) - \rho_N(g)\overline{\Phi}^t(x)}_2 + \norm{\widehat{\Phi}^t_M(\rho_0(g)x) - \rho_N(g)\overline{\Phi}^t(x)}_2\nonumber \\
            &\qquad \qquad =\norm{\rho_N(g)\widehat{\Phi}^t_M(x) - \rho_N(g)\overline{\Phi}^t(x)}_2 + \norm{\widehat{\Phi}^t_M(\rho_0(g)x) -\overline{\Phi}^t(\rho_0(g)x)}_2.
        \end{align}
        Hence, if 
        \begin{align} \label{eq:individual_point}
            \norm{\widehat{\Phi}^t_M(y) -\overline{\Phi}^t(y)}_2\leq r
        \end{align}
        for $r= \varepsilon/2$ and every $y\in O_x$, we obtain the stated equivariance bound. For a fixed $y$, bounding the probability of failure of \eqref{eq:individual_point} is simple: applying the Matrix Bernstein inequality \cite{tropp2015introductionmatrixconcentrationinequalities} to $(\Phi_{A^t}(x)-\overline{\Phi}^t(x))/M$, we get that for every $y\in O_x$
        \begin{align}\label{eq:bernstein}
            \mathbb{P}\Big(\norm{\widehat{\Phi}^t_M(y)-\overline{\Phi}^t(y)}_2\geq r\Big)\leq (d+1)\cdot \exp{\Bigg(\frac{-r^2/2}{\frac{V_{t}}{M}+\frac{2r\cdot \xi_t}{3 M}}\Bigg)}.
        \end{align}
        In case 1, we can now take a union bound over $O_x$ for the Inequality \ref{eq:bernstein}. We get the following bound:
        \begin{align*}
            \mathbb{P}\Big(\exists y\in O_x,\,\norm{\widehat{\Phi}^t_M(y)-\overline{\Phi}^t(y)}_2\geq r\Big)\leq \abso{O_x}\cdot(d+1)\cdot \exp{\Bigg(\frac{-r^2/2}{\frac{V_{t}}{M}+\frac{2r\cdot \xi_t}{3 M}}\Bigg)}.
        \end{align*}
        Thus, by choosing $M\geq \ln{\Big(\frac{\abso{O_x}\cdot(d+1)}{\delta}\Big)}\cdot\Big(\frac{2V_t}{r^2}+\frac{4\xi_t/3}{r}\Big) = C \cdot\max \Big(\frac{V_{t}}{r^2}, \frac{\xi_t}{r}\Big)\log\Big(\frac{d\vert{O_x}\vert}{\delta}\Big)$, we guarantee that
        \begin{align*}
            \mathbb{P}\Big(\norm{\widehat{\Phi}^t_M(y)-\overline{\Phi}^t(y)}_2 \leq r, \, \forall y\in O_x\Big) \geq 1-\delta.
        \end{align*}

        As for case 2, let us WLOG assume that $K\geq 1$ -- the other case is exactly the same. Let $Q$ be an $\tfrac{\varepsilon}{3K}$- net for $O_x$. By the above argument, we get that \eqref{eq:individual_point} for $r=\tfrac{\varepsilon}{3K}$ and every $y\in Q$ with a failure probability smaller than stated. This is however enough to get the inequality with $r=\varepsilon$ on the entire orbit: First, notice that both $\widehat{\Phi}_M^t$ and $\overline{\Phi}^t$ also are $K$-Lipschitz continuous as averages of $K$-Lipschitz continuous functions. Now, if $z\in O_x$ is arbitrary, we can argue that there exists a point $y\in Q$ with $\norm{z-y}\leq \tfrac{\epsilon}{3K}$, and consequently
        \begin{align*}
             \norm{\widehat{\Phi}^t_M(z) -\overline{\Phi}^t(z)}_2 &\leq  \norm{\widehat{\Phi}^t_M(z) -\widehat{\Phi}^t_M(y)}_2 +  \norm{\widehat{\Phi}^t_M(y) -\overline{\Phi}^t(y)}_2\ +  \norm{\overline{\Phi}^t(y) -\overline{\Phi}^t(z)}_2 \\
             &\leq K \cdot \tfrac{\epsilon}{3K} + K\cdot\tfrac{\epsilon}{3K}  + K\cdot \tfrac{\epsilon}{3K} \leq \varepsilon,
        \end{align*}
        where we in the final estimate used $K\geq 1$. The claim follows.
    \end{proof}
    
    The above bounds of course do not say much without an estimate of $\vert{O_x}\vert$ and $\calN(O_x,\tfrac{\varepsilon}{3K})$. As for the former, an obvious upper bound is $\abso{O_x}\leq \abso{G}$, yielding the statement that an ensemble with $O(\log(\vert{G}\vert))$ members is enough to obtain approximate equivariance for a given element $x$, which in terms of the group size is a relatively mild growth. In special cases, $\vert O_x \vert$ can be much smaller than $\vert G \vert$. An example is the standard action of the permutation group acting on $\R^n$: Then, $\vert{G}\vert=n!$, but $\vert O_x \vert \leq n$.

    A natural use case for the infinite orbits is the case of $G$ being a connected Lie group. In that case, the following proposition shows that the logarithm of the covering number does not grow faster than the dimension of $G$. 

    \begin{proposition}[Bound on the covering number of the orbit]
        Let $G$ be a connected lie Group. Then, we can bound
        \begin{align*}
            \log \calN\left(O_x, \frac{\varepsilon}{3\max(K,1)}\right) \leq \dim(G) \cdot \log\left(1+\tfrac{1}{\varepsilon}\cdot\max(K,1) \cdot C(\mathrm{diam}(G),\rho_0) \cdot \norm{x}_2 \right),
        \end{align*}
        where $C(\mathrm{diam}(G),\rho)$ is a constant dependent on the diameter of the group and the representation $\rho_0$.
    \end{proposition}
    \begin{proof}
        Let us begin by recalling some basic facts about compact, connected Lie groups (see e.g. \cite{fulton2004representation}). Let $\mathfrak{g}$ denote the Lie algebra of $G$ and $\exp: \mathfrak{g} \to G$ the exponential map. Since $G$ is a connected, compact Lie group, we know that
        \begin{align}
            G = \{\exp(H) \, \vert \, \norm{H}\leq \mathrm{diam}(G)\}.
        \end{align} 
        Furthermore, since $\rho_0$ is a representation, we have $\underline{\exp}(\mathrm{d}\rho_0(H))= \rho_0(\exp(H))$, where $\mathrm{d}\rho_0:\mathfrak{g} \to \mathfrak{gl}(X)$ is the differential of $\rho_0:G \to \mathrm{GL}(X)$ at $e$, and $\underline{\exp}$ on the left is the exponential map $\mathfrak{gl}(X)\to GL(X)$:
        \begin{align*}
            \underline{\exp}(A) = \sum_{k=0}^\infty \tfrac{1}{k!}A^k
        \end{align*}
        It is well known that $\underline{\exp}$  is differentiable, with a derivative that fulfills
        \begin{align*}
            \norm{\underline{\exp}'(A)B}_{2\to 2} \leq \exp(\norm{A}_{2\to 2})\norm{B}_{2\to 2}.
        \end{align*}
        Hence, for $H,H'\in \mathfrak{g}$ with $\norm{H},\norm{H'}\leq \mathrm{diam}(G)$, we get
        \begin{align*}
            &\norm{\rho_0(\exp(H))x-\rho_0(\exp(H'))x}_2 \leq  \norm{\underline{\exp}(\mathrm{d}\rho_0(H))-\underline{\exp}(\mathrm{d}\rho_0(H')}_{2\to 2} \norm{x}_2 \\
            &\qquad \qquad \leq \sup_{\norm{\widetilde{H}}\leq \mathrm{diam}(G)} \exp(\norm{\mathrm{d}\rho_0(\widetilde{H})}_{2\to 2}) \cdot \norm{\mathrm{d}\rho_0(H))-\mathrm{d}\rho_0(H')}_{2\to 2} \cdot \norm{x}_2 \\
            &\qquad \qquad \leq \sup_{\norm{\widetilde{H}}\leq \mathrm{diam}(G)} \exp(\norm{\mathrm{d}\rho_0(\widetilde{H})}_{2\to 2}) \cdot \norm{\mathrm{d}\rho_0}\cdot \norm{H-H'}\cdot \norm{x}_2 \\
&\qquad \qquad =: C(\mathrm{diam}(G), \rho_0) \cdot \norm{H-H'} \cdot \norm{x}_2       \end{align*}
    
        From this inequality, we quickly derive that if $P$ is a net of width $\tfrac{\varepsilon}{3\cdot\max( K,1)}~\cdot~{C(\mathrm{diam}(G), \rho_0) \norm{x}_2)}$ of the ball of radius $\mathrm{diam}(G)$ in $\mathfrak{g}$, the set $\{\rho_0(\exp(H))x \, \vert H \in P\}$ is a net of width $\varepsilon/(3\max(K,1))$ of the orbit. It is well known that such a net $P$ can be constructed of cardinality less than 
        \begin{align*}
            \left(1+\tfrac{2\cdot\mathrm{diam}(G)\cdot 3\max(K,1) \cdot C(\mathrm{diam}(G),\rho_0) \cdot \norm{x}_2}{\varepsilon}\right)^{\mathrm{\dim}(G)}.
        \end{align*}
        (see e.g. the appendix of \cite{FouRau2013}). We here used that $\dim \mathfrak{g}=\dim G$. The claim follows.
    \end{proof}
The moral of the above result is that when $G$ is a Lie-group,  $O(\dim(G))$ ensemble members are needed to realise equivariance approximately (where the implicit constant depends on $\rho_0$ and $\mathrm{diam}(G)$). 

\section{Numerical experiments}
\label{sec:exp}
We perform two small numerical experiments. Their purpose is to test the robustness of our theorems to using the sample mean instead of the true ensemble mean, the necessity of the $\rho$-invariance assumption of $\calL$, and  the $\rho$-invariance of the initial distribution of parameters. The code, that relies heavily of the code made available from \cite{nordenfors2024optimizationdynamicsequivariantaugmented}, can be found 
at \href{https://github.com/onordenfors/ensemble_experiment}{https://github.com/onordenfors/ensemble\_experiment}
%in the supplementary materials,
%at \href{https://github.com/onordenfors/ensemble_experiment}{github.com/onordenfors/ensemble\_experiment},
and additional technical details (such as hardware specifications, licenses for the datasets, etc.) in Appendix \ref{app:experiment}.

 We use two metrics to evaluate the invariance of the trained models. First, similarly to \cite{gerken2024emergentequivariancedeepensembles}, 
%In order to test invariance of predicted label to rotations,
we calculate the average \emph{orbit same prediction} (OSP). That is, for each test image, we test how many of the transformed versions of it are given the same prediction as the untransformed one by the ensemble model. OSP lies between $1$ and $\abso{G}$, where \text{a value equal to the order of the group $\abso{G}$} indicates perfect invariance. % is know as the and has a minimum value of $1$ (because the identity element in $C_4$ acts by the identity transformation) and a maximum value of the size of the orbit ($4$ in the case of the rotation action of $C_4$)
The second metric we calculate is the average symmetric KL-divergence $D_{\mathrm{KL}}$ between the class probabilities predicted on \text{an} example with each of the \text{transformed} versions of it. Note that the class probabilities can vary quite a lot but still yield the same prediction, whence this is a finer measure of invariance than the OSP. 
%The code for the experiment was written in Python, using the PyTorch library.

\subsection{Experiment 1: Discrete rotations of images}
We train ensembles of 1000 small CNNs %of the same structure that is used in \cite{nordenfors2024optimizationdynamicsequivariantaugmented}:
with three convolutional layers of size 16 followed by one fully connected one, with layer normalization \cite{ba2016layer} and $\tanh$ non-linearities. See also Figure \ref{fig:architecture}. We train families of nearly identical ensembles, the only difference being that the support of the filters: the members of one ensemble have filters of symmetric support, whereas those of the other have asymmetric support. The models are trained for $10$ epochs on the MNIST dataset \cite{lecun1998a}, using SGD with a constant learning rate of $0.01$, a batch size of 32, and cross entropy loss. We apply random augmentation, and test the ensembles on both MNIST and CIFAR-10 \cite{krizhevsky09}.

\begin{figure}[!h]
        \centering
        \begin{tikzpicture}[scale=0.38]
            \draw (0,0) rectangle (1,3);
            \node at (0.5,1.5) {$X$};
            \node at (0.5,4) {$\rho^{\mathrm{rot}}$};
            \node at (0.5,-1) {$\R^{1\times 28\times 28}$};
            \node at (4,1) {$\tanh$, \scriptsize{LN}};
            \node at (4,2) {Conv, Pool,};
            \draw (7,0) rectangle (8,3);
            \node at (7.5,1.5) {$X_1$};
            \node at (7.5,-1) {$\R^{16\times14\times 14}$};
            \node at (11,1) {$\tanh$, \scriptsize{LN}};
            \node at (11,2) {Conv, Pool,};
            \draw (14,0) rectangle (15,3);
            \node at (14.5,1.5) {$X_2$};
            \node at (14.5,-1) {$\R^{16\times7\times 7}$};
            \node at (18,1) {$\tanh$, \scriptsize{LN}};
            \node at (18,2) {Conv,};
            \draw (21,0) rectangle (22,3);
            \node at (21.5,1.5) {$X_3$};
            \node at (21.5,-1) {$\R^{16\times7\times 7}$};
            \node at (25,1) {\scriptsize{FC}};
            \node at (25,2) {Flatten,};
            \draw (28,0) rectangle (29,3);
            \node at (28.5,1.5) {$Y$};
            \node at (28.5,-1) {$\R^{10}$};
            \node at (28.5,4) {$\rho^{\mathrm{triv}}$};
            \draw[->] (1.5,1.5) -- (6.5,1.5);
            \draw[->] (8.5,1.5) -- (13.5,1.5);
            \draw[->] (15.5,1.5) -- (20.5,1.5);
            \draw[->] (22.5,1.5) -- (27.5,1.5);
        \end{tikzpicture}
        
        \caption{The architecture used in our neural networks. The convolutions have filters with support as in Figure \ref{fig:supports} (left or right). LN stands for LayerNorm and FC stands for Fully-Connected.}
        \label{fig:architecture}
\end{figure}

\paragraph{Case 1: $C_4$} We first choose $C_4$, i.e. rotations of multiples of $90^\circ$, as our symmetry group. For this $G$, $\calLsym$ is $\rho$-invariant, but $\calLas$ is not (this is actually true for all possible choices of the intermediate representations -- see Appendix \ref{app:skew}). Hence,  we expect that the former ensembles become equivariant, whereas the latter do not. We initialize the ensembles by drawing the coefficients $c$ from a standard normal distribution (see Lemma \ref{lem:gaussinit}). Since $\calLas$ is not $\rho$-invariant, this does not produce an invariant initial distribution. We can mitigate this by setting the 'corner pieces' of all asymmetric filters equal to zero. We hence get three ensembles: One symmetric, one asymmetric initialized $\rho$-invariantly, and one asymmetric initialized ''naïvely''.

In Figure \ref{fig:C4} (top), we plot the values of our two metrics on the two datasets at the end of epoch 10, for different number of ensemble members. We see that for all models, the ensembles become more equivariant as the number of members grows, showcasing the inherent 'equivariazing' effect of ensembling well. On MNIST, the larger ensembles come close to a perfect OSP of $4$. However, it can also be observed that ensembles of a size as small as 50 have an OSP above 3.5 \emph{on out-of-distribution data}. We also see that the symmetric ensembles outperform the asymmetric ones initialized invariantly, which in themselves outperform the naïvely initialized asymmetric ones.  This suggests that $\rho$-invariance of the architecture also plays an important role empirically. The differences are more prominent on the OOD data, which is to be expected -- the networks have actually 'seen' rotated nines in the MNIST data and can hence have learned to predict them correctly even if the network is not inherently equivariant -- the same is not true for rotated cars in the CIFAR10 set.

More detailed results are given in Appendix \ref{app:tables}. We there in particular showcase the evolution of the ensembles during training there, in Figure \ref{fig:cross_vs_skew} (there are no surprises - the ensembles are quite far from exact equivariance at their completely random initialized state, but become essentially as equivariant as after ten epochs already after one epoch).  %We see that with respect to both metrics, as is expected, the symmetric ensembles s. . %for MNIST test data can be seen in Figure \ref{fig:cross_vs_skew}.

\begin{figure}[!h] \centering
    \includegraphics[width=0.6\linewidth]{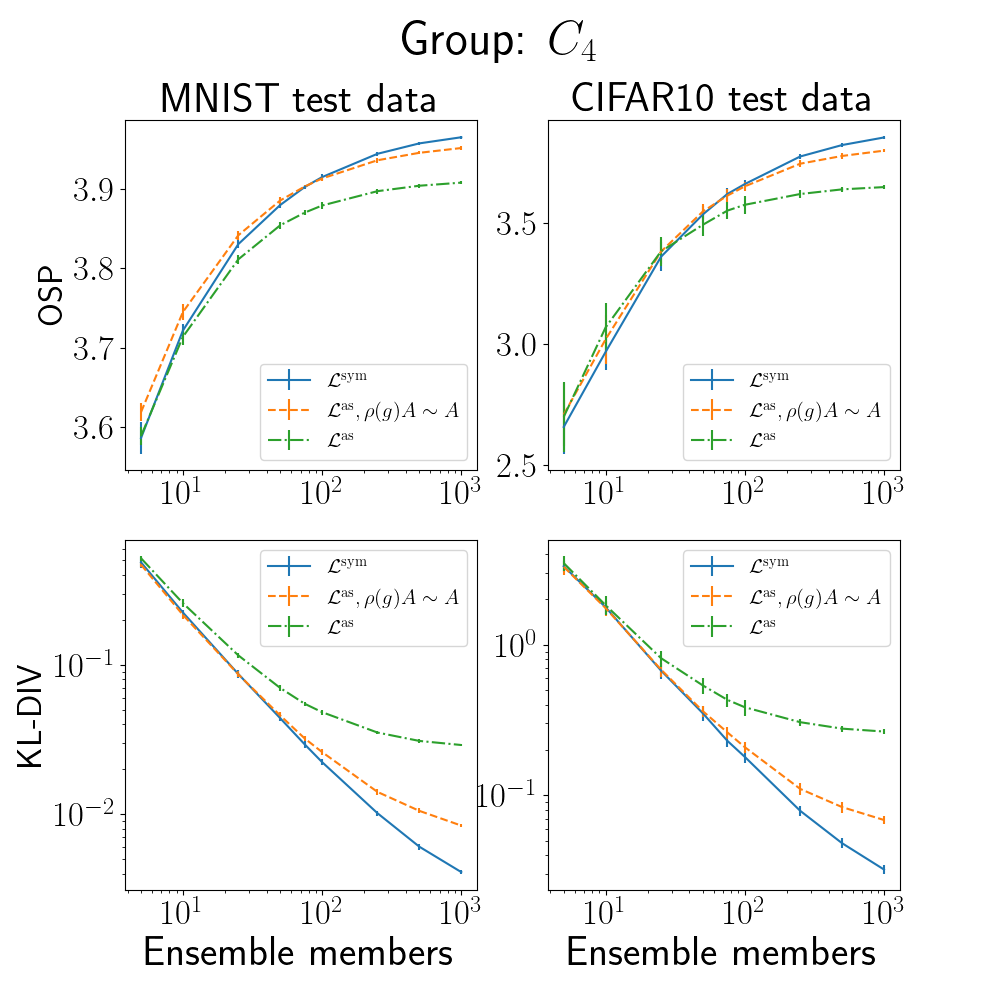} \\
    \includegraphics[width=0.6\linewidth]{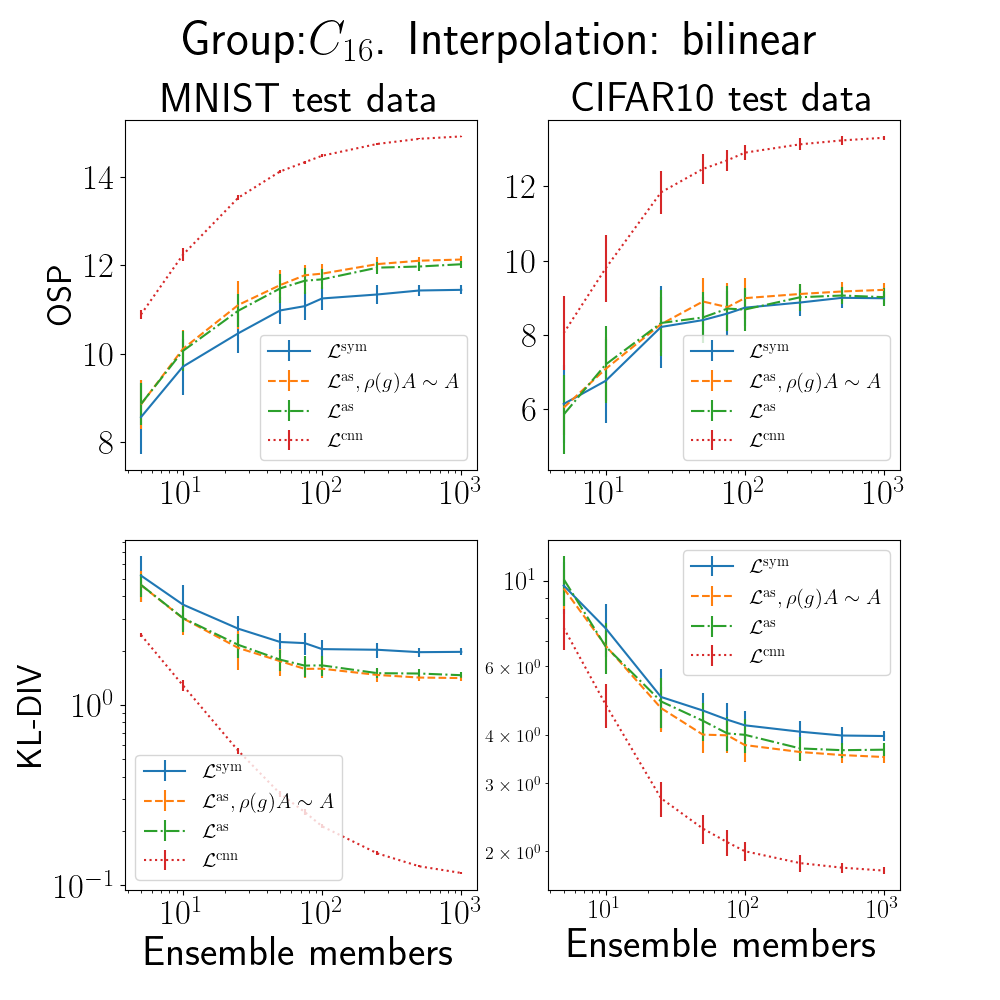}
 \caption{Metrics after the 10\ts{th} epoch for different ensemble sizes for the $C_4$ experiment (top) and $C_{16}$ experiment (bottom). Each datapoint is a mean of 30 bootstrapped examples -- the errorbars denotes one standard deviation of the bootstrap.  The $x$-scale in the top plots are logarithmic, both scales are logarithmic in the bottom plots. Best viewed in color. \label{fig:C4}}
\end{figure}

\paragraph{Case 2: $C_{16}$} We perform a second set of experiments with $C_{16}$, i.e., rotations of multiples of 22.5$^\circ$, as the symmetry group. One should note that in this setting, neither $\calL^{\mathrm{sym}}$ nor $\calL^{\mathrm{assym}}$ are $\rho$-invariant --  rotating a symmetric filter by 45 degrees will send it to a filter supported on the corners of the $3\times 3$ square, and so on. We therefore include a fourth model in these experiments, namely a standard, full $3\times3$-support CNN. Note that due to interpolation effects, the corresponding space $\calL^{\mathrm{cnn}}$ is also not perfectly $C_{16}$ invariant, but at most approximately. We here use a bilinear interpolation -- results for a 'nearest' interpolation, as well as an illustration of the non-invariance issues, are given in Appendix \ref{app:nearest}. The experimental setup is exactly as above (except a batch size of $128$, to speed up training).

The performance with respect to the different metrics for the four models at epoch $10$ is presented in Figure \ref{fig:C4} (right) (more details again in Appendix \ref{app:experiment}). Compared to the $C_4$ experiment, all models fare much worse -- no model comes close to the optimal OSP of $16$. Note that is what to be expected from our theory -- the compatibility assumption is not fulfilled for any model. We also see that the standard CNN:s outperform the other models vastly -- which is also what could be expected, given that the $3\times3$-square is closer to being an invariant support than the symmetric and asymmetric supports are.

\paragraph{The necessity of $\rho$-invariance.} In our theoretical results, $\rho$-invariance of $\calL$ is shown to be \emph{sufficient} to achieve equivariant ensembles. What does our experiments say about the \emph{necessity} of it? This is a subtle point, that requires a nuanced response. As already discussed, the results indicate the \emph{relevance} of the condition. In both of our experiments, the cases with (more) invariant $\calL$-spaces outperformed the other ones. On the other hand, the ensembles with non-invariant $\calL$ are also surprisingly equivariant. In the $C_{16}$ experiments, the $CNN$-models achieve an OSP of over $13$ on out-of-distribution data, and the different restricted architectures around $9$, although neither $\calL^{\mathrm{cnn}}$, nor the other spaces, are perfectly invariant. This is not close to the perfect $16$, but also far from the value of $2.5$ one expects from a completely random model. 

We can at this point only speculate on why ensembles with non-invariant $\calL$ become as equivariant as they do. A possible explanation is revealed by our proofs: There we need $\rho$-invariance to guarantee $\rho(g)\Pi_{\calL}=\Pi_\calL\rho(g)$ for all $g$. In all of our cases, $\rho(g)\Pi_\calL$ is still ''almost'' equal to $\Pi_\calL\rho(g)$ for many $g$. For instance, in the $C_4$-case, $\Pi_{\calL}\rho(g)=\Pi_{\calL} \rho(g)$ only fails due to one non-zero corner in the $\calL^{\mathrm{asym}}$-filters. We investigate this point further in Appendix \ref{app:approxerror}, where we perform $C_4$-experiments with $5\times 5$-filters, which contain more energy in the asymmetric part. Our results there support this hypothesis somewhat, but more work is needed to draw definitive conclusions.

% although the asymmetric ensembles are less invariant than the symmetric ones, they are still remarkably close to being invariant. This cannot be completely explained by our theory. Note that we assume $\rho$-invariance of $\calL$ is order to achieve that $\rho(g)\Pi_{\calL}=\Pi_\calL\rho(g)$. However, in our case, $\rho(g)\Pi_\calL$ is still ''almost'' equal to $\Pi_\calL\rho(g)$ for the asymetric filters -- it only differs for the non-zero corners. We investigate this point further in Appendix \ref{app:approxerror}. Our results there support this hypothesis somewhat, but more work is needed in the future.

%Now let us finally comment on the fact that even the \emph{non-invariantly initialized} ensembles seem to become more equivariant after training. Note that the main result of our paper can be interpreted as a 'stationarity' result -- if we start invariantly distributed, we stay that way. Our experiment indicates that the the invariant distributions even are \emph{attractors}. If this is true, and if so how generally, is an interesting direction of future work.

\subsection{Experiment 2: Continuous rotations of point clouds}
In another set of experiments, we train 1000 PointNets \cite{qi2017pointnet,qi2017pointnet++} on the ModelNet10 dataset \cite{wu20153d}, using the Pytorch Geometric library \cite{fey2019fast}. PointNets are message passing networks built for processing graph data, with small MLPs as message passing functions. We again use LayerNormalization and $\tanh$ activation, and use a fully connected linear layer on top of a PointNet layer to classify point clouds. The invariance of $\calL$ essentially follows from the fact that all components of the network are MLPs, and the discussion in Example \ref{ex:prepost}. A more detailed description, along with an argument for why it corresponds to an $\calL$ invariant under the natural action of $\mathrm{SO}(3)$ on point clouds in $\R^3$, can be found in Appendix \ref{app:pndetails}.

\begin{figure}[!h]
    \begin{minipage}[c]{.5\textwidth}
        \includegraphics[width=.9\linewidth]{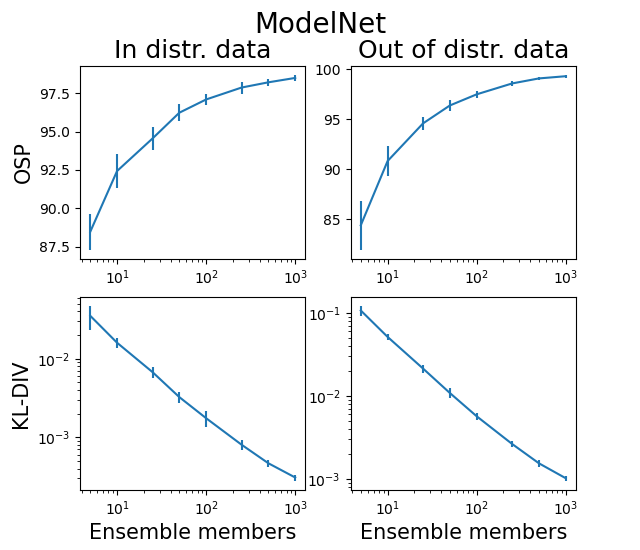}
    \end{minipage}
    \begin{minipage}[c]{.45\textwidth}
         \caption{Metrics after the 10\ts{th} epoch for different ensemble sizes for the ModelNet experiment. Each datapoint is a mean of 30 bootstrapped examples -- the errorbars denotes one standard deviation of the bootstrap.  The $x$-scale in the top plots are logarithmic, both scales are logarithmic in the bottom plots.\label{fig:ModelNet}}
    \end{minipage}
\end{figure}

We proceed in large as above: after training 1000 PointNets for 10 epochs with SGD with learning rate of 5e-3, we bootstrap 30 ensembles for different sizes and record the metrics on both in- and out-of-distribution data. The in-distribution data in this case is the test data set of ModelNet10, while the out of distribution data are test data from ten classes of ModelNet40 not included in ModelNet10, namely \emph{airplane, bench, bookshelf, bottle, bowl, car, cone, cup, curtain, door}.

For each dataset, we calculate the same metrics as in Experiment 1, that is the OSP and symmetric KL-divergence. However, since $\SO(3)$ is infinite, we of course cannot use \emph{all} transformations to calculate these. We instead use 100 uniformly random sampled ones to calculate the metrics. Since none of the rotations is the identity (with probability 1), the OSP is between 0 and 100 in this case. The results can be viewed in Figure \ref{fig:ModelNet}. It should here be mentioned that the individual models suffer from overfitting, which partially explains why smaller ensembles are already quite equivariant. We discuss this in Appendix \ref{app:pndetails7}. However, this does not change the fact that ensembling \emph{increases} the equivariance, which is what the theory predicts. A t-test ($p<.001$) reveals that the difference in equivariance is significant -- all metrics are significantly better than all others for the 1000-member ensembles.

\subsection*{Conclusion}
%\label{sec:con]

In this paper, we have shown that ensembles of neural networks become equivariant when trained with mini-batch SGD on randomly augmented data. This result generalizes previous works in that they apply to finite-width neural networks, and more importantly, to more general architecture than just MLPs. In regards to the latter, our analysis revealed that the geometry of the architecture plays a role. When their architecture spaces fulfill a geometric invariance condition, the equivariance of the ensembles follow.

\subsection*{Acknowledgment}
\label{sec:ack}
This work was partially supported by the Wallenberg AI, Autonomous Systems and
Software Program (WASP) funded by the Knut and Alice Wallenberg Foundation. The computations were enabled by resources provided by the National Academic Infrastructure for Supercomputing in Sweden (NAISS), partially funded by the Swedish Research Council through grant agreement no. 2022-06725. We also thank the anonymous reviewers of previous versions of this manuscript, whose comments have helped to improve it.

%We succeeded in our goal with Theorem \ref{thm:main}, where we proved that, modulo an oft--satisfied $\rho$-invariance condition on the network architecture, ensembles become $\rho$-equivariant when training with SGD and random augmentation, assuming that the parameters of the network were distributed $\rho$-invariantly to begin with (e.g., with Gaussian initialization). That is, we have shown that 
%if the goal %is a one wishes to train a neural network model to be 
%is a group equivariant model, it suffices to train a neural network ensemble model under data augmentation. Our numerical experiments, while not conclusive, support these conclusions for sample mean ensembles as well.

\begin{appendices}

\section{Convolutions with asymmetric filters and \texorpdfstring{$C_4$}{C4}-invariance}
\label{app:skew}
In the main paper, we have used the 'convolution-spaces' $\calLsym$ and $\calLas$ as running examples. We have argued that if the canonical representation $\rhorot$ is used on all intermediate spaces, $\calLsym$ is invariant under all transformations $\rho(g)$, whereas $\calLas$ is not. However, as we remarked in Example \ref{ex:convcond}, there is no inherent reason why the actions on the intermediate spaces should always be $\rhorot$. Hence, there might exist intermediate representations so that $\calLas$ is $\rho$-invariant. The purpose of this section is to show that if we fix the representation $\rhorot$ on $X$, there is not. To do this, it is clearly enough to focus on the space of convolutional maps between $X_0=\R^{n,n}$ and $X_1= (\R^{n,n})^d$.

To make the case clear that this is an important and nontrivial point, let us first show that there are representations other than $\rhorot$ on $X_1$ which makes $\calLsym$ invariant. Th this end, let $\rhoind$ be a representation of $C_4$ on $\R^d$. Such exist: Note that $C_4$ is isomorphic to $\Z_4$, the integers equipped with addition modulo $4$, and $\Z_4$ naturally acts on the space $\R^4$ by shifting the entries, i.e $(\varrho(k)v)_\ell = v_{\ell-k}$, which then also does $C_4$. In obvious ways, this can be extended to any $\R^d$. For any representation $\rhoind$ of $C_4$ on $\R^d$, we can combine it  with $\rhorot$ as follows:
\begin{align} \label{eq:prodrep}
    \big((\rhoind \odot \rhorot)(g) x\big)_k = \bigg( \sum_{\ell\in [d]}\rhoind(g)_{k\ell} \rhorot(g) x_\ell\bigg)_k.
\end{align}
In words, $\rhoind \odot \rhorot$ first rotates each entry of a tuple $x$ according to $\rhoind(g)$, and then transforms the $d$ resulting $\R^{N,N}$-images by $\rhoind$ as if they were entries in a $\R^d$-vector. A straightforward calculation  now shows that the lifted representation $\widehat{\rho_0}(g)$ then maps $\sprod{\psi}$ to
\begin{align} \label{eq:liftedrep}
    \widehat{\rho_0}(g)\sprod{\psi}=\sprod{\big(\sum_{\ell} \rhoind(g)_{k\ell} \rhorot(g)\psi_\ell\big)_k}.
\end{align}
Consequently, if the filters $\psi_\ell$ are symmetric filters, the transformed ones are also. On the other hand, if the $\psi$ are supported on the asymmetric support, the linear combinations in \eqref{eq:liftedrep} will in general not be, and hence $\rho(g)\calL\not\subseteq \calL$ also for intermediate representations of the form \eqref{eq:prodrep}.

The last point opens up a route to prove that there are no representations $\rho_1$ at all that makes $\calLas$ invariant under the corresponding lifted representation: If we show that if  $\rho_1$ is a representation for which
    $\rho(g)\calLas\sse \calLas$ for all $g$, then it must be as in \eqref{eq:prodrep}, then we are by the previous discussion done. This is the purpose of the following theorem.

\begin{theorem}[Asymmetric CNNs are not invariant under action of $C_4$]
    Let $N\geq 3$, and let $C_4$ act through $\rhorot$ on $X$. If $\rho_1$ is a representation such that $\calLas$ is $\rho$-invariant under the lifted representation, it must be as in  \eqref{eq:prodrep}. Since no such representations make $\calLas$ $\rho$-invariant, there are  no representations of $C_4$ on $X_1$ that make $\calLas$ $\rho$-invariant.
\end{theorem}
\begin{proof}
    %To ease the notation, let us write $C_\varphi$ instead of $(C_{\varphi_i}x)_{i\in [d]}$, use the notation $\rhorot$ for the $d$-fold direct product of $\rhorot$, and in particular $C_{\rhorot(g)\varphi} =(C_{\rhorot(g)\varphi_i}x)_{i\in [d]}$. In this notation, \dots says that  $\rhorot(g)\sprod{\varphi} \rhorot(g)^{-1} = \sprod{\rhorot(g)\varphi}$. Consequently,

    For any representation $\rho_1$ on $X_1$, we have
    \begin{align*}
       \rho(g)\sprod\varphi = \rho_1(g)\sprod\varphi \rhorot(g)^{-1} &= \rho_1(g)\rhorot(g)^{-1}\rhorot(g)\sprod\varphi \rhorot(g)^{-1} \\&=\rho_1(g)\rhorot(g)^{-1}\sprod{\rhorot(g)\varphi}.
    \end{align*}
    Now, we assume that $\rho_1(g)$ is a representation for which $\rho(g)\calL \sse \calL$. Then, the above operator or all $\varphi$ and $g$ must be of the form $\sprod{\psi}$  for some filter $\psi$. Since convolutions commute with any translation $T_\ell$, we can then argue that
    \begin{align}
        T_\ell \rho_1(g)\rhorot(g)^{-1} \sprod{\rhorot(g)\varphi} &= T_\ell \sprod{\psi} = \sprod{\psi} T_\ell =  \rho_1(g)\rhorot(g)^{-1}\sprod{\rhorot(g)\varphi}T_\ell \nonumber \\
        &= \rho_1(g)\rhorot(g)^{-1}T_\ell \sprod{\rhorot(g)\varphi} \label{eq:prototranslommutation}
        %&= \rho_1(g)T_{g^{-1}\ell}\rhorot(g)^{-1} C_{\rhorot(g)\varphi} %
    \end{align}
    where we again did not distinguish between translation representation on $X$ and the direct product of them on $X_1$. Since this equality is true for any $\varphi$, this implies that 
    \begin{align}
          T_\ell \rho_1(g)\rhorot(g)^{-1} = \rho_1(g)\rhorot(g)^{-1}T_\ell\label{eq:translommutation}
    \end{align}
    for all $g$ and $\ell$. A more technical argument goes as follows: By choosing $\varphi$ equal to the tuple with only a non-trivial filter in the $i$:th channel, that filter to have only one non-zero pixel, and subsequently evaluating \eqref{eq:prototranslommutation} on a basis of $X$, we get the desired equality of operators on a basis of $X_1$.

    Now, \eqref{eq:translommutation} simply means that the operator $\rho_1(g)\rhorot(g)^{-1}$ commutes with translations for every $g$. It is well known that this is  equivalent to $\rho_1(g)\rhorot(g)^{-1}$ being a convolution operator $\sprod{\chi(g)}$ for some filters $\chi_{k\ell}(g)$, $k, \ell \in [d]$. Inserting this form into the definition of $\rho$ yields
    \begin{align}
        \rho(g)\sprod\varphi = \sprod{\chi(g)} \sprod{\rhorot(g)\varphi} = \bigg\langle \sum_{\ell \in [d]} \chi_{k\ell}(g)*(\rhorot(g)\varphi_\ell) \label{eq:newfilter}\bigg \rangle 
    \end{align}
     Now, suppose that any filter $\chi_{k\ell}(g)$ has a support of more than one pixel. Then, since $N \geq 3$, we can construct a filter $\varphi_\ell$ supported on the asymmetric set so that the support of $\chi_{k\ell}(g)*(\rhorot(g)\varphi_\ell)$ is not -- essentially, the convolution by $\chi_{k\ell}(g)$ would spread out the support of $\varphi_\ell$.  By letting all other filters in a tuple be zero, we would conclude that the filter
     \begin{align*}
       \sum_{\ell \in [d]} \chi_{k\ell}(g)*(\rhorot(g)\varphi_\ell)
     \end{align*}
     has a support which is not contained in the asymmetric $\Omega_0$. Together with \eqref{eq:newfilter}, this means that $\rho(g)\sprod{\varphi}\notin \calLas$, which would be a contradiction. Hence, all $\sprod{\chi_g}$ must be multiples of the identity for all $g$, which means that 
     \begin{align*}
         [\rho_1(g)v]_k = \sum_{\ell \in [d]} c_{k\ell}(g)\rhorot(g)x_{\ell}
     \end{align*}
     for some numbers $c_{k\ell}(g)$. It is now only left to show that the $c_{k\ell}$ define a representation. We however have for $g,h\in C_4$ arbitrary
     \begin{align*}
         [\rho_1(gh)x]_k &= \sum_{\ell \in [d]} c_{k\ell}(gh)\rhorot(gh)x_{\ell} \\
         [\rho_1(g)\rho_1(h)x]_k& = \sum_{\ell \in [d]} c_{k\ell}(g)\rhorot(g)(\rho_1(g)x)_{\ell} = \sum_{\ell,m \in [d]}c_{km}(g)\rhorot(g)c_{m\ell}(h)\rhorot(h)x_{\ell}
     \end{align*}
     Since $\rho_1$ and $\rhorot$ are representations, the above expressions are equal for every $x$ and $g$, which means that
     \begin{align*}
     \forall k, \ell \in [d]: \quad      c_{k\ell}(gh) = \sum_{m \in [d]}c_{km}(g)c_{m\ell}(h)
     \end{align*}
     This is however only another way of saying that the matrices $C(g)=(c_{k\ell}(g))_{k,\ell}$ fulfill $C(gh)=C(g)C(h)$, which is what was to be shown.
\end{proof}

\section{Discrete Rotation experiment details}

\label{app:experiment}
\subsection{Hardware}
The $C_4$-experiments were performed on a cluster using NVIDIA Tesla T4 GPUs with 16GB RAM per unit. The total compute time for training all 3000 individual models for $10$ epochs each is estimated at $\sim 115$ hours. The $C_{16}$-experiments were made on the same cluster, but instead using NVIDIA Tesla A40 GPUs with 64GB RAM per unit. These experiments (\texttt{BILINEAR} and \texttt{NEAREST}) used a total of $\sim 800$ compute hours. This estimate does not include test runs during debugging etc.

\subsection{Licenses}
MNIST is available under an CC BY-SA 3.0 licence. CIFAR10 does not have a formal license, but is publically available at \url{https://www.cs.toronto.edu/~kriz/cifar.html}.

\subsection{Architecture}
The architecture used for all three networks is the same with the exception of the support of the $3\times 3$ filters. Namely, in the first layer we use a convolution with zero padding, with $1$ channel in and $16$ out, followed by an average pooling with a $2\times 2$ window with a stride of $2$, followed by a $\tanh$ activation, followed by a layer normalization. The second layer is the same except the convolution has $16$ channels instead of $1$. The third layer is the same as the second, but without the average pooling. The final layer flattens the image channels into a vector, followed by a linear transformation into $\R^{10}$.

All of the nonlinearities involved here are equivariant to $C_4$- rotations, and so do not interfere with the equivariance/invariance of the network \text{in that case}. Due to interpolation effects, it is not equivariant to $C_{16}$ -- adding to the assumptions of our theorem that are violated in that case.

%In terms of inputs and outputs, the first layer transforms an MNIST sample, which has dimensions $1 \times 28\times 28$ into a feature of dimensions $16\times 14\times 14$. The second layer then transforms this into a feature of dimensions $16\times 7\times 7$. The third layer then transforms this into a feature of dimensions $16\times 7\times 7$ once again. The final layer then flattens this into a $784\times 1$ vector and applies a $10\times 784$ matrix to it which yields the output in $\R^{10}$.

A prediction of the network is then the $\mathrm{argmax}$ of the output, which yields the predicted label as a one-hot vector.
%\subsection{Symmetric Kullback-Leibler Divergence}
%In the definition of the Kullback-Leibler divergence $KL(P\,||\,Q) = \sum_x P(x)(\log P(x)-\log Q(x))$, $P$ and $Q$ have to be probability distributions. We thus apply a $\mathrm{softmax}$ function to the outputs from the ensemble, to turn them into probability distributions. The symmetric Kullback-Leibler divergence is simply $KL(P\,||\,Q)+KL(Q\,||\,P)$.

\subsection{Evolution of the equivariance during training} \label{}

\begin{figure}[!h]
\begin{minipage}[c]{.74\textwidth}
    \centering
    \includegraphics[width=.9\textwidth]{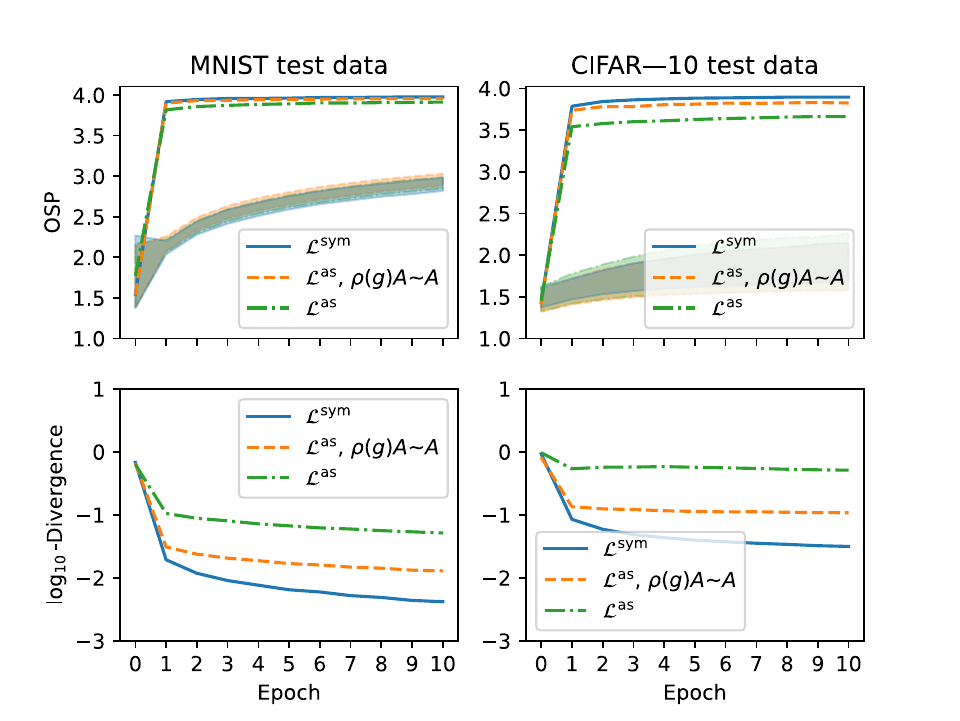}
\end{minipage}
\begin{minipage}[c]{.25\textwidth}
\caption{Top: OSP for ensembles with $1000$ members, for the $C_4$-experiments.\\
(Higher is better)\\
The middle $95\%$ of individual ensemble members are within the shaded area. \\
Bottom: Logarithm of symmetric Kullback--Leibler Divergence for ensembles with $1000$ members. \\
(Lower is better).\\Best viewed in color.}\label{fig:cross_vs_skew}
\end{minipage}
\end{figure}

 As advertised in the main paper, we here present the evolution of the metrics along the training for the $C_4$ experiment in Figure \ref{fig:cross_vs_skew}. We see that although the models become more equivariant with increased training, they practically become as equivariant as at the end of training already after one epoch. Note that the low OSP at the $0$:th epoch, i.e. at initialization, actually is to be expected. At initialization, the weights are completely random, which is therefore also to be expected for the predictions. A network with random predictions is expected to give an OSP of $1.75$, which is close to the OSP that we observe. The largeness of the other metric can mainly be attributed to the finiteness of the ensemble -- when evaluated for ensembles of size $10$ and $100$, we get values of $1.13$ and $0.58$, respectively, which is much larger than the $-0.17$ we observe for the $1000$ member ensembles. It is plausible that this trend of decreasing numbers with increasing ensemble size would continue with even more ensemble members.

\subsection{Detailed results for different ensemble sizes} \label{app:tables}
We here present our results in table form: the $C_4$-experiments in Table \ref{tab:C4}, $C_{16}$-experiments with $\texttt{BILINEAR}$ interpolation in Table \ref{tab:C16bilinear} and $C_{16}$-experiments with $\texttt{NEAREST}$ interpolation in Table \ref{tab:C16nearest}. We indicate models that perform statistically significantly better than all others (according to a $t$-test of the 30 bootstrapped examples) than all others of the experiments ($p<.001$) with bold font.

\begin{table}[h]
\begin{tabular}{|c|c|c|c|c|c|c|c|c|c|c|c|} 
 \hline & Metric & Model & 5 & 10 & 25 & 50 & 75 & 100 & 250 & 500 & 1000 \\ 
\hline
\multirow{6}{*}{\rotatebox[origin=c]{90}{MNIST}}
 & \multirow{3}{*}{OSP}& $\mathcal{L}^{\mathrm{sym}}$& 3.587& 3.722& 3.830& 3.879& 3.902& 3.914& \textbf{3.944}& \textbf{3.957}& \textbf{3.964}\\ 
& &  $\mathcal{L}^{\mathrm{as}}$ (s.in) & \textbf{3.619}& \textbf{3.745}& \textbf{3.841}& \textbf{3.885}& 3.903& 3.913& 3.936& 3.945& 3.951\\ 
& &  $\mathcal{L}^{\mathrm{as}}$& 3.589& 3.714& 3.811& 3.854& 3.870& 3.879& 3.897& 3.904& 3.907\\
  \cline{2-12}& \multirow{3}{*}{$\log D_{\mathrm{KL}}$}& $\mathcal{L}^{\mathrm{sym}}$& -0.31& -0.65& -1.06& -1.36& \textbf{-1.53}& \textbf{-1.65}& \textbf{-1.99}& \textbf{-2.22}& \textbf{-2.39}\\ 
& & $\mathcal{L}^{\mathrm{as}}$  (s.in) & -0.33& -0.67& -1.06& -1.34& -1.49& -1.58& -1.85& -1.98& -2.08\\ 
& & $\mathcal{L}^{\mathrm{as}}$& -0.29& -0.59& -0.94& -1.15& -1.26& -1.32& -1.45& -1.51& -1.54\\ 
\hline
\multirow{6}{*}{\rotatebox[origin=c]{90}{CIFAR-10}}  & \multirow{3}{*}{OSP}& $\mathcal{L}^{\mathrm{sym}}$& 2.656& 2.967& 3.361& 3.536& 3.620& 3.661& \textbf{3.775} & \textbf{3.823}& \textbf{3.854}\\ 
& & $\mathcal{L}^{\mathrm{as}}$  (s.in) & 2.707& 3.019& 3.382& 3.548& 3.614& 3.652& 3.746& 3.778& 3.800\\ 
& & $\mathcal{L}^{\mathrm{as}}$& 2.698& 3.068& 3.383& 3.493& 3.552& 3.576& 3.620& 3.639& 3.649\\ 
\cline{2-12} & \multirow{3}{*}{$\log D_{\mathrm{KL}}$}& $\mathcal{L}^{\mathrm{sym}}$& 0.52& 0.25& -0.17& -0.46& \textbf{-0.64}& \textbf{-0.74}& \textbf{-1.10}& \textbf{-1.32}& \textbf{-1.49}\\ 
& & $\mathcal{L}^{\mathrm{as}}$  (s.in) & 0.52& 0.24& -0.17& -0.44& -0.58& -0.68& -0.96& -1.08& -1.16\\ 
& &  $\mathcal{L}^{\mathrm{as}}$& 0.54& 0.26& -0.09& -0.27& -0.37& -0.42& -0.51& -0.56& -0.58\\ 
\hline 
 \end{tabular}
    \caption{Results for the $C_4$ experiment. Shown are the metrics measured for different ensemble sizes at the last epoch of training. (s.in) refers to the symmetrical initialization. The numbers presented are means of 30 bootstrapped ensembles of the respective sizes. Bold results that are significantly better (as measured by a $t$-test, $p<.001$) than all other models at the respective metric and size.}
    \label{tab:C4}
\end{table}

\begin{table}[h]
\begin{tabular}{|  c | c | c |c | c |c |c|c|c|c|c|c|c |c|} 
 \hline &Metric & Model & 5 & 10 & 25 & 50 & 75 & 100 & 250 & 500 & 1000 \\ 
\hline \multirow{8}{*}{\rotatebox[origin=c]{90}{MNIST}} 
&\multirow{4}{*}{OSP}& $\mathcal{L}^{\mathrm{sym}}$& 8.56& 9.70& 10.46& 10.98& 11.07& 11.25& 11.34& 11.43& 11.44\\ 
& & $\mathcal{L}^{\mathrm{as}}$ (s.in) & 8.84& 10.10& 11.10& 11.55& 11.77& 11.81& 12.03& 12.10& 12.13\\ 
& & $\mathcal{L}^{\mathrm{as}}$& 8.86& 10.06& 10.97& 11.48& 11.65& 11.68& 11.95& 11.97& 12.02\\ 
& & $\mathcal{L}^{\mathrm{cnn}}$& \textbf{10.89}& \textbf{12.24}& \textbf{13.53}& \textbf{14.13}& \textbf{14.34}& \textbf{14.48}& \textbf{14.75}& \textbf{14.87}& \textbf{14.92}\\ 
\cline{2-12} & \multirow{4}{*}{$\log D_{\mathrm{KL}}$} & $\mathcal{L}^{\mathrm{sym}}$& 0.72& 0.56& 0.42& 0.35& 0.34& 0.31& 0.31& 0.29& 0.29\\ 
& & $\mathcal{L}^{\mathrm{as}}$(s.in)& 0.66& 0.48& 0.32& 0.24& 0.20& 0.20& 0.17& 0.15& 0.15\\ 
& & $\mathcal{L}^{\mathrm{as}}$& 0.66& 0.48& 0.33& 0.25& 0.22& 0.22& 0.18& 0.17& 0.16\\ 
& & $\mathcal{L}^{\mathrm{cnn}}$& \textbf{0.39}& \textbf{0.11}& \textbf{-0.25}& \textbf{-0.49}& \textbf{-0.59}& \textbf{-0.67}& \textbf{-0.82}& \textbf{-0.89}& \textbf{-0.93}\\ 
\hline \multirow{8}{*}{\rotatebox[origin=c]{90}{CIFAR-10}} 
 & \multirow{4}{*}{OSP}& $\mathcal{L}^{\mathrm{sym}}$& 6.14& 6.76& 8.21& 8.40& 8.58& 8.73& 8.87& 9.00& 8.99\\ 
& & $\mathcal{L}^{\mathrm{as}}$ (s.in)& 6.04& 7.08& 8.29& 8.90& 8.75& 8.99& 9.10& 9.17& 9.21\\ 
& & $\mathcal{L}^{\mathrm{as}}$& 5.86& 7.20& 8.32& 8.47& 8.71& 8.68& 9.01& 9.06& 9.02\\ 
& & $\mathcal{L}^{\mathrm{cnn}}$& \textbf{8.05}& \textbf{9.79}& \textbf{11.84}& \textbf{12.47}& \textbf{12.70}& \textbf{12.90}& \textbf{13.13}& \textbf{13.24}& \textbf{13.31}\\ 
\cline{2-12} & \multirow{4}{*}{$\log D_{\mathrm{KL}}$}& $\mathcal{L}^{\mathrm{sym}}$& 0.99& 0.88& 0.70& 0.66& 0.64& 0.63& 0.61& 0.60& 0.60\\ 
& & $\mathcal{L}^{\mathrm{as}}$ (s.in)& 0.98& 0.83& 0.67& 0.60& 0.60& 0.57& 0.56& 0.55& 0.54\\ 
& & $\mathcal{L}^{\mathrm{as}}$& 1.00& 0.83& 0.69& 0.64& 0.60& 0.60& 0.57& 0.56& 0.56\\ 
& & $\mathcal{L}^{\mathrm{cnn}}$& \textbf{0.88}& \textbf{0.68}& \textbf{0.44}& \textbf{0.36}& \textbf{0.32}& \textbf{0.30}& \textbf{0.27}& \textbf{0.26}& \textbf{0.25}\\ 
\hline 
 \end{tabular}
    \caption{Results for the $C_{16}$ experiment with \texttt{BILINEAR} interpolation. Shown are the metrics measured for different ensemble sizes at the last epoch of training. (s.in) refers to the symmetrical initialization. The numbers presented are means of 30 bootstrapped ensembles of the respective sizes. Bold results that are significantly better (as measured by a $t$-test, $p<.001$) than all other models at the respective metric and size}
    \label{tab:C16bilinear}
\end{table}

\section{PointNet experiment details} \label{app:pndetails}

\subsection{Hardware}
The training was performed on a cluster using NVIDIA Tesla T4 GPUs with 16GB RAM per unit. The total compute time for training all 1000 individual models for $10$ epochs each is estimated at $\sim 45$ hours. We then generated model outputs for 100 rotations of test data from two different datasets on the same cluster, but instead using NVIDIA Tesla A40 GPUs with 64GB RAM per unit. This inference used an estimated total of $\sim 60$ compute hours. These estimates do not include test runs during debugging etc.
\subsection{Licenses}
ModelNet 10 and ModelNet 40 do not have any formal license, but are both available for research purposes from \hyperlink{https://modelnet.cs.princeton.edu/}{https://modelnet.cs.princeton.edu/}.

\subsection{Architecture description}
PointNet layers process data defined on graphs. Graphs are sets of nodes $[n]$, out of which some are connected via edges. The edge structure can either be encoded as a set $E$ of pairs of indices, or by an adjacency matrix $A\in \R^{n,n}$. The layer is parametrized by two functions -- the local function $h$ and the global function $\gamma$. Features $x_i \in \R^m$ and positions $p_i\in \R^d$ of the points in the graph are processed by first getting sent individually through the local function, then averaged over neighborhoods $N_i$ in the graph, and then being sent through the global function, to produce a new set of features:
\begin{align*}
    x_i = \gamma\left( \sum_{j \in N_i } h(x_i,p_j-p_i)\right)
\end{align*}
Notice here that we chose to use average aggregation instead of the default $\max$-aggregation to keep the layer differentiable.

The functions $h$ and $\gamma$ are free to choose. In our experiments, $h$ is a fully connected neural network of two layers of size $103$ and $200$, and $\gamma$ is a fully connected network with two layers of size $200$ and $300$. We use layer normalization between each pair of layers in the two neural networks, and choose $\tanh$ as the activation function. See also Figure \ref{fig:architecture_pn_components}.

\begin{figure}[!h]
        \centering
        \begin{tikzpicture}[scale=0.38]
            \draw (0,0) rectangle (1,3);
            \node at (-2,1.5) {\LARGE $h$};
            %\node at (0.5,1.5) {$X$};
            %\node at (0.5,4) {$\rho^{\mathrm{rot}}$};
            \node at (0.5,-1) {$\R^{100}\times \R^3$};
            \node at (4,1) {$\tanh$, \scriptsize{LN}};
            \node at (4,2) {Fully conn.};
            \draw (7,0) rectangle (8,3);
            %\node at (7.5,1.5) {$X_1$};
            \node at (7.5,-1) {$\R^{103}$};
            \node at (11,1) {$\tanh$, \scriptsize{LN}};
            \node at (11,2) {Fully conn};
            \draw (14,0) rectangle (15,3);
            %\node at (14.5,1.5) {$X_2$};
            \node at (14.5,-1) {$\R^{200}$};
            %\node at (18,1) {$\tanh$, \scriptsize{LayerNorm}};
            %\node at (18,2) {Fully conn.};
            %\draw (21,0) rectangle (22,3);
            %\node at (21.5,1.5) {$X_3$};
            %\node at (21.5,-1) {$\R^{200}$};
            %\node at (25,1) {\scriptsize{Fully-Connected}};
            %\node at (25,2) {Flatten,};
            %\draw (28,0) rectangle (29,3);
            %\node at (28.5,1.5) {$Y$};
            %\node at (28.5,-1) {$\R^{10}$};
            %\node at (28.5,4) {$\rho^{\mathrm{triv}}$};
            \draw[->] (1.5,1.5) -- (6.5,1.5);
            \draw[->] (8.5,1.5) -- (13.5,1.5);
            %\draw[->] (15.5,1.5) -- (20.5,1.5);
            %\draw[->] (22.5,1.5) -- (27.5,1.5);

            \draw (0,-6) rectangle (1,-3);
            \node at (-2,-4.5) {\LARGE $\gamma$};
            %\node at (0.5,1.5) {$X$};
            %\node at (0.5,4) {$\rho^{\mathrm{rot}}$};
            \node at (0.5,-7) {$\R^{200}$};
            \node at (4,-5) {$\tanh$, \scriptsize{LN}};
            \node at (4,-4) {Fully conn.};
            \draw (7,-6) rectangle (8,-3);
            %\node at (7.5,-4.5) {$X_1$};
            \node at (7.5,-7) {$\R^{200}$};
            \node at (11,-5) {$\tanh$, \scriptsize{LN}};
            \node at (11,-4) {Fully conn};
            \draw (14,-6) rectangle (15,-3);
            %\node at (14.5,-4.5) {$X_2$};
            \node at (14.5,-7) {$\R^{300}$};
            %\node at (25,1) {\scriptsize{Fully-Connected}};
            %\node at (25,2) {Flatten,};
            %\draw (28,0) rectangle (29,3);
            %\node at (28.5,1.5) {$Y$};
            %\node at (28.5,-1) {$\R^{10}$};
            %\node at (28.5,4) {$\rho^{\mathrm{triv}}$};
            \draw[->] (1.5,-4.5) -- (6.5,-4.5);
            \draw[->] (8.5,-4.5) -- (13.5,-4.5);
            %\draw[->] (15.5,-4.5) -- (20.5,-4.5);
            
            \draw (0,-12) rectangle (1,-9);
            \node at (-2,-10.5) {\LARGE $\kappa$};
            %\node at (0.5,-4.5) {$X$};
            %\node at (0.5,4) {$\rho^{\mathrm{rot}}$};
            \node at (0.5,-13) {$\R^{300}$};
            \node at (4,-11) {$\tanh$};
            \node at (4,-10) {Fully conn.};
            \draw (7,-12) rectangle (8,-9);
            %\node at (7.5,-10.5) {$X_1$};
            \node at (7.5,-13) {$\R^{200}$};
            \node at (11,-11) {$\tanh$};
            \node at (11,-10) {Fully conn};
            \draw (14,-12) rectangle (15,-9);
            %\node at (14.5,-10.5) {$X_2$};
            \node at (14.5,-13) {$\R^{100}$};
            \node at (18,-10) {Fully conn.};
            \draw (21,-12) rectangle (22,-9);
            %\node at (28.5,1.5) {$Y$};
            \node at (21,-13) {$\R^{10}$};
            %\node at (28.5,-8) {$\rho^{\mathrm{triv}}$};
            \draw[->] (1.5,-10.5) -- (6.5,-10.5);
            \draw[->] (8.5,-10.5) -- (13.5,-10.5);
            \draw[->] (15.5,-10.5) -- (20.5,-10.5);
            \draw (0,-18) rectangle (1,-15);
            \node at (-2,-16.5) {\LARGE $
            \epsilon$};
            %\node at (0.5,-4.5) {$X$};
            %\node at (0.5,4) {$\rho^{\mathrm{rot}}$};
            \node at (0.5,-19) {$\R^{3}$};
            \node at (4,-17) {$\tanh$};
            \node at (4,-16) {Fully conn.};
            \draw (7,-18) rectangle (8,-15);
            %\node at (7.5,-10.5) {$X_1$};
            \node at (7.5,-19) {$\R^{50}$};
            \node at (11,-17) {$\tanh$};
            \node at (11,-16) {Fully conn};
            \draw (14,-18) rectangle (15,-15);
            %\node at (14.5,-10.5) {$X_2$};
            \node at (14.5,-19) {$\R^{50}$};
            \node at (18,-16) {Fully conn.};
            \draw (21,-18) rectangle (22,-15);
            %\node at (28.5,1.5) {$Y$};
            \node at (21,-19) {$\R^{100}$};
            %\node at (28.5,-8) {$\rho^{\mathrm{triv}}$};
            \draw[->] (1.5,-16.5) -- (6.5,-16.5);
            \draw[->] (8.5,-16.5) -- (13.5,-16.5);
            \draw[->] (15.5,-16.5) -- (20.5,-16.5);
        \end{tikzpicture}
        
        \caption{The local $h$, global $\gamma$, 'head' $\kappa$, and $\epsilon$ embedding neural networks in our architecture. LN stands for LayerNorm.}
        \label{fig:architecture_pn_components}
\end{figure}

The entire net consists of three components. First, the positions $p_i$ are (nodewise) fed through an embedding network $\epsilon$ (fully connected, three layers of size $50$, $50$ and $100$) to define features $x_i$. These are then together with the $p_i$ ran through one PointNet convolution layer as above. These are then averaged over the graph and sent through a third three-layer MLP  $\kappa$ with layers of size $200$, $100$ and $10$. This MLP also uses the $\tanh$ non-linearity, but no layer normalization. See also Figure \ref{fig:architecture_pn_components}.

\subsection{The architecture and the \texorpdfstring{$\calL$}{L}-formalism} Let us  comment on how the PointNet convolution can be incorporated in the framework described in the paper, i.e. understood as an affine subset of all possible fully connected networks, and that the corresponding set of maps is $\SO(3)$-invariant. The construction is in essence the same as the one for general message passing in Appendix A of \cite{nordenfors2024optimizationdynamicsequivariantaugmented}, with the additional complication of multiple-layer $h$ and $\gamma$.  We only include it for completeness. We also formulate all neural network layers here as not having biases -- as is described \cite{nordenfors2024optimizationdynamicsequivariantaugmented}, such can routinely be included again by adding a 'dummy-dimension' to the input vectors and then describe the affine-linear maps as linear ones.

We understand the input of the network as triples of feature sequences, position sequences and adjacency matrix, say $\calU = (\R^{m})^{n}\oplus (\R^{d})^n \oplus \R^{n,n}$. The embedding network $\epsilon$ only acts non-trivially on the features and by the identity on the positions and adjacency matrix. Thus, the fully-connected layers of $\epsilon$ maps from $\calU$ to $\calU_1=(\R^{k})^{n}\oplus (\R^{d})^n \oplus \R^{n,n}$. The first layer of the $h$-network maps these onto a $n\times n$-array of features in, say $\R^{m_1}$ -- one vector per edge (to emulate the dependence of $h(x_i,p_j-p_i)$ on the relative position $(p_j-p_i)$) through some linear maps:
\begin{align}
    L(x,p)_{j,i} = Mx_i + N(p_j-p_i).
\end{align}
The linear maps obtained by varying $M$ and $N$ in $\R^{m_1,m+d}$ in this ways clearly forms a subspace of all linear maps between $\calU_1$ and $(\R^{m_1})^{n \times n}$. To be able to access the adjacency structure, we 'save it' by applying the identity mapping to it. To conclude, the first layer of the $h$ net is hence a linear map from $\calU_1$ to $\calU_2 = (\R^{m_1})^{n\times n} \oplus \R^{n,n}$ of the form
\begin{align*}
    (x,p,A) \mapsto (L(x,p),A).
\end{align*}
We then apply a non-linearity pointwise to the 'edge-feature', $(y,A) \mapsto (\sigma(y),A)$, thereafter another linearity of the form
\begin{align*}
    (y,A) \mapsto (Ly,A), \quad (Ly)_{k,j} = My_{k,j}, M\in \R^{m_{i+1},m_i},
\end{align*}
and so on. 

The next step is the averaging step, which we can interpret as a nonlinearity from $(\R^{m_{L}})^{n\times n} \times \R^{n,n}\to (\R^{m_L})^n$
\begin{align*}
    (y,A) \mapsto y_i = \sum_j a_{ij}y_{ij}
\end{align*}
The global net can then be treated just as above, where we now do not even need to worry about 'saving' the adjacency structure anymore. The admissible layers are again characterized by the fact that they act pointwise on each node $(Ly)_i = My_i, M\in \R^{m_{i+1},m_i}$. The final averaging and 'head-MLP' are also easy to treat.

If we equip the input space with the canonical action of $\SO(3)$, i.e. $\rho^{\mathrm{in}}(g)(x,p)_i = (gx_i,gp_i)$ (remember that we use the positions also as features), and all other layers with the trivial representation, the above structure describes a $G$-invariant space $\calL$. This follows essentially as in Example \ref{ex:prepost}: for all but the very first layer, the lifted representation $\rho$ is trivial, so that invariance of those parts are trivial. The first layer of $\epsilon$ is still a fully-connected layer after applying the lifted representation, which deals with the rotations of the features. As for the first layer of $h$, it is not hard to see that the lifted representation here maps the admissible layer defined by $M$ and $N$ to the one defined by $M, Ng^T$, so that it is in particular still admissible.

\subsection{Detailed results for different ensemble sizes}
The results from the point cloud experiment are presented in Table \ref{tab:modelnet}

\begin{table}[h]
\begin{tabular}{|c|c|c|c|c|c|c|c|c|c|c|}
\hline  Dataset & Metric & 5 & 10 & 25 & 50 & 100 & 250 & 500 & 1000 \\
 \hline \multirow{2}{*}{In dist.} & OSP & 88.48 & 92.43 & 94.55 & 96.23 & 97.09 & 97.87 & 98.21 & 98.49 \\
\cline{2-10} &  $\log(\text{KL-DIV})$ & -1.45 & -1.80 & -2.17 & -2.49 & -2.76 & -3.09 & -3.33 & -3.51\\
 \hline \multirow{2}{*}{Out of dist.} & OSP  & 84.41 & 90.83 & 94.60 & 96.39 & 97.50 & 98.59 & 99.10 & 99.32 \\
\cline{2-10} & $\log(\text{KL-DIV})$ & -0.97 & -1.29 & -1.67 & -1.97 & -2.25 & -2.58 & -2.81 & -2.99 \\
 \hline \end{tabular}
 \caption{\label{tab:modelnet} Mean metrics after 10 epochs. Each number is the mean of 30 bootstrapped ensembles of the respective size.}
 \end{table}

\subsection{A comment on the high equivariance of small ensembles}\label{app:pndetails7}
When training the PointNets on augmented data, the random rotations made the learning task too difficult for our simple training procedure. Thus, many of the trained models suffer from overfitting, predicting classes 2 and 7 for most samples, since these two classes are overrepresented in the dataset. However, this does not interfere with the point of the experiment, which is to test the equivarizing effect of ensembling. It would only be a problem for the experiment if the models predicted only a single class for every sample, since this would immediately make it perfectly equivariant.
\section{Additional experiments}

\begin{table}[h]
\begin{tabular}{|c|c|c|c|c|c|c|c|c|c|c|c|c|c|c|} 
 \hline & Metric & Model & 5 & 10 & 25 & 50 & 75 & 100 & 250 & 500 & 1000 \\ 
\hline \multirow{6}{*}{\rotatebox[origin=c]{90}{MNIST}} & \multirow{3}{*}{OSP}& $\mathcal{L}^{\mathrm{sym}}$& 11.02& 12.42& 13.67& 14.25& \textbf{14.52}& \textbf{14.67}& \textbf{14.99}& \textbf{15.11}& \textbf{15.19}\\ 
& & $\mathcal{L}^{\mathrm{as}}$ (s.in)& 11.15& 12.50& 13.68& 14.24& 14.46& 14.57& 14.83& 14.93& 14.98\\ 
& & $\mathcal{L}^{\mathrm{as}}$& 11.18& 12.54& 13.71& 14.24& 14.46& 14.59& 14.84& 14.94& 15.00\\ 
\cline{2-12}  & \multirow{3}{*}{$\log D_{\mathrm{KL}}$}& $\mathcal{L}^{\mathrm{sym}}$& 0.28& -0.02& \textbf{-0.40}& \textbf{-0.66}& \textbf{-0.80}& \textbf{-0.89}& \textbf{-1.12}& \textbf{-1.24}& \textbf{-1.31}\\ 
& & $\mathcal{L}^{\mathrm{as}}$ (s.in)& 0.29& -0.01& -0.37& -0.61& -0.73& -0.79& -0.97& -1.04& -1.08\\ 
& & $\mathcal{L}^{\mathrm{as}}$& 0.28& -0.01& -0.37& -0.61& -0.72& -0.80& -0.97& -1.05& -1.09\\ 
\hline \multirow{6}{*}{\rotatebox[origin=c]{90}{CIFAR-10}} & \multirow{3}{*}{OSP}& $\mathcal{L}^{\mathrm{sym}}$& 4.76& 5.24& 5.14& 5.32& 5.41& 5.32& 5.47& 5.48& 5.50\\ 
& & $\mathcal{L}^{\mathrm{as}}$ (s.in)& 4.90& 5.29& 5.33& 5.47& 5.49& 5.56& 5.54& 5.57& 5.57\\ 
& & $\mathcal{L}^{\mathrm{as}}$& 4.76& 5.07& 5.51& 5.56& 5.55& 5.57& 5.60& 5.59& 5.62\\ 
\cline{2-12} & \multirow{3}{*}{$\log D_{\mathrm{KL}}$}& $\mathcal{L}^{\mathrm{sym}}$& 1.11& 1.02& 0.98& 0.94& 0.92& 0.93& 0.91& 0.91& 0.91\\ 
& & $\mathcal{L}^{\mathrm{as}}$ (s.in)& 1.09& 0.99& 0.93& 0.90& 0.88& 0.87& 0.86& 0.86& 0.86\\ 
& & $\mathcal{L}^{\mathrm{as}}$& 1.10& 1.01& 0.91& 0.88& 0.88& 0.86& 0.85& 0.85& 0.85\\ 
\hline 
 \end{tabular}
    \caption{Results for the $C_{16}$ experiment with \texttt{NEAREST} interpolation. Shown are the metrics measured for different ensemble sizes at the last epoch of training. The numbers presented are means of 30 bootstrapped ensembles of the respective sizes. Bold results that are significantly better (as measured by a $t$-test, $p<.001$) than all other models at the respective metric and size.}
    \label{tab:C16nearest}
\end{table}

\subsection{\texorpdfstring{$C_{16}$}{C16} with \texorpdfstring{\texttt{NEAREST}}{NEAREST} interpolation} \label{app:nearest}
\begin{figure}[!h]
    \centering
    \includegraphics[width=0.85\linewidth]{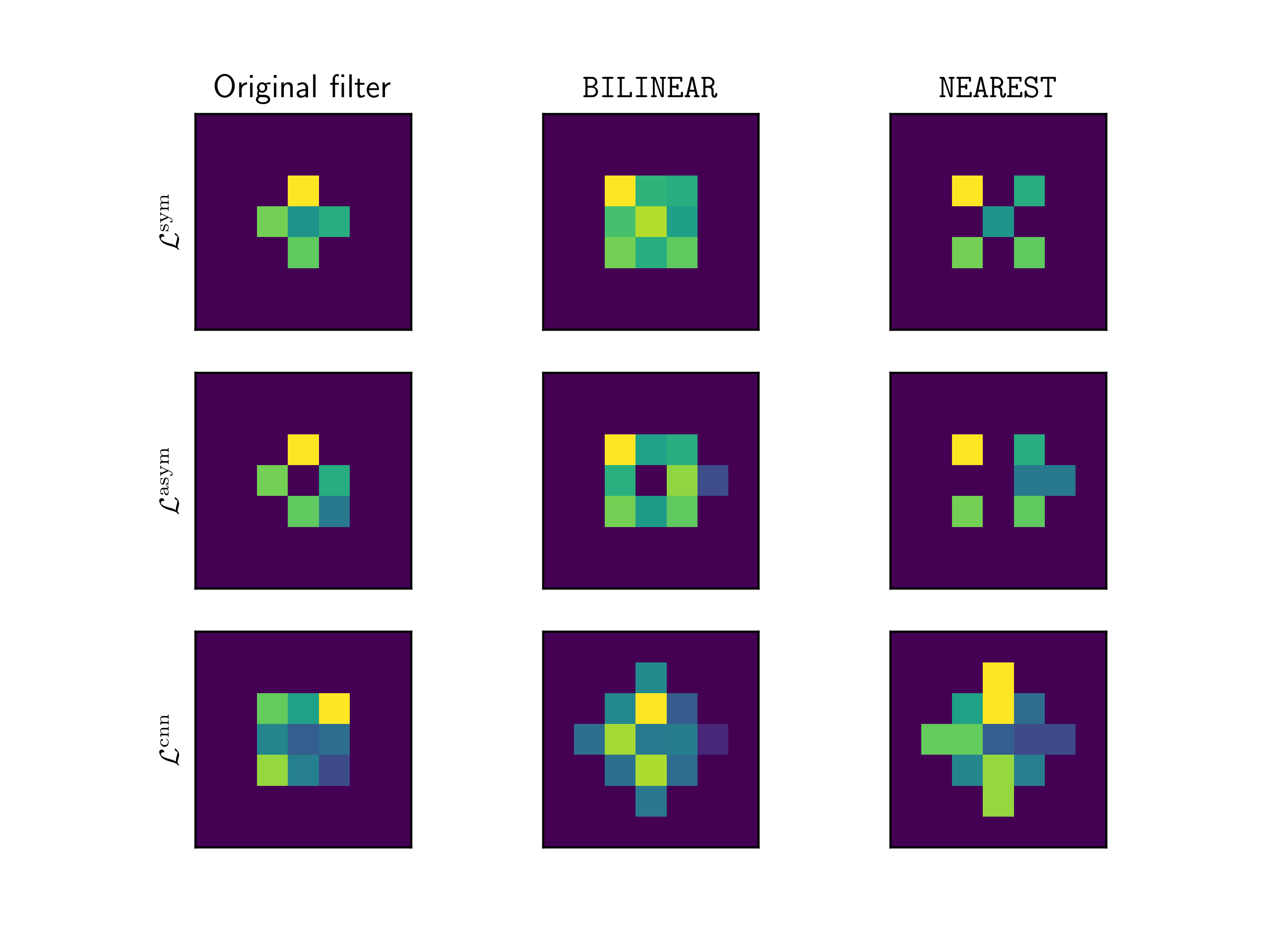}
    \caption{The effects of different types of interpolation on different types of filters. Filters in $\calLsym$, $\calLas$ and $\calL^{\mathrm{cnn}}$ are rotated 45 degrees using the $\texttt{BILINEAR}$ and $\texttt{NEAREST}$ interpolations.}
    \label{fig:filter_effects}
\end{figure}

We repeat the experiment in the main paper while using \texttt{NEAREST} interpolation scheme. The results are given in Figure \ref{fig:C16NEAREST}. Again, none of the models fare particularly well, as the theory suggests. Compared to the \texttt{BILINEAR} interpolation experiment, see Figure \ref{fig:C4} (right), we see that the models fare better on the MNIST data, but worse on the CIFAR data. We think that the reason for the better performance of the in-distribution data is that the \texttt{NEAREST} interpolation introduces fewer 'bluriness' artefacts than the \texttt{BILINEAR} one -- the former still operates by permuting pixels. This means that the test set formed by applying \texttt{BILINEAR}-rotations to MNIST results in a more diverse dataset -- which makes it simpler to fit to it.

We speculate that the reason for the worse performance on the CIFAR data is that the the operators $\Pi_\calL$ and $\rho_i(g)$ are even further from commuting in the $C_{16}$ case compared to the $C_4$ case -- see Figure \ref{fig:filter_effects} . 

\begin{figure}[!h]
    \begin{minipage}{.6\textwidth}
    \centering
    \includegraphics[width=0.85\linewidth]{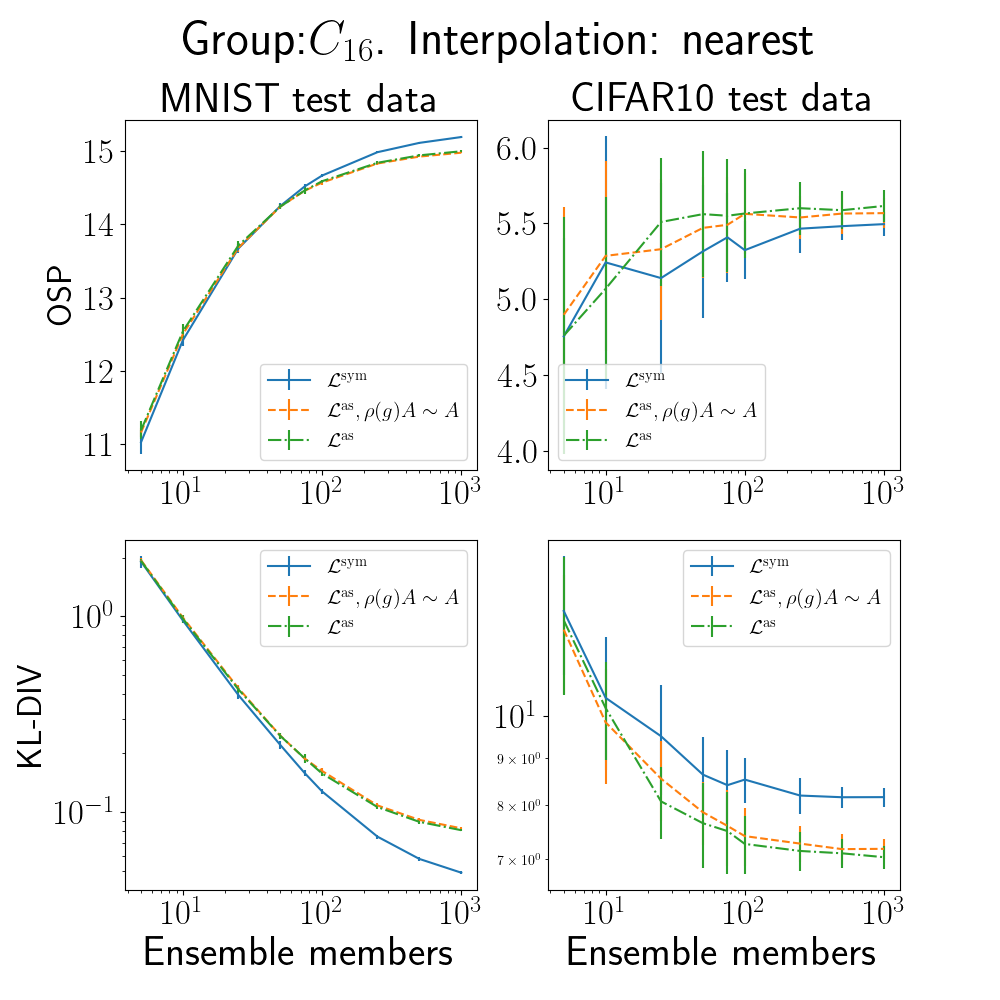}
    \end{minipage}
    \begin{minipage}{.35\textwidth}
    \caption{Metrics after the 10th epochs for different ensemble sizes for the $C_{16}$ experiment using the \texttt{NEAREST} interpolation. \\ Each datapoint is a mean of 30 bootstrapped examples -- the errorbars denotes one standard deviation of the bootstrap. \\ The $x$-scale in the top plots are logaritmic, both scales are logaritmic in the bottom plots. Best viewed in color. \label{fig:C16}}
    \label{fig:C16NEAREST}
    \end{minipage}
\end{figure}

\subsection{Experiment with larger filters}\label{app:approxerror}
We perform the same experiment as in Section \ref{sec:exp} with convolution filters of size $5\times 5$. The reason for performing this experiment is to test whether the 'small' error in the approximation $\rho(g)\Pi_\calL A\approx \Pi_\calL\rho(g) A$ in the case of $3\times 3$-filters is responsible for the small difference in performance between symmetric and asymmetric models. In a larger filter we can have more energy in the asymmetric parts of the filter, so we want to see if this affects the results. We consider therefore two CNNs with filters as in Figure \ref{fig:supports5x5}.
\begin{figure}[!h]
\begin{minipage}{.69\textwidth}
        \centering
        \begin{tikzpicture}[scale=0.5]
            \draw[step=1cm] (0,0) grid (5,5);
            \draw[fill=gray!50] (1,1) rectangle (2,2);
            \draw[fill=gray!50] (3,3) rectangle (4,4);
            \draw[fill=gray!50] (1,3) rectangle (2,4);
            \draw[fill=gray!50] (3,1) rectangle (4,2);
            \draw[fill=gray!50] (2,0) rectangle (3,1);
            \draw[fill=gray!50] (2,1) rectangle (3,2);
            \draw[fill=gray!50] (2,3) rectangle (3,4);
            \draw[fill=gray!50] (2,4) rectangle (3,5);
            \draw[fill=gray!50] (0,2) rectangle (1,3);
            \draw[fill=gray!50] (1,2) rectangle (2,3);
            \draw[fill=gray!50] (2,2) rectangle (3,3);
            \draw[fill=gray!50] (3,2) rectangle (4,3);
            \draw[fill=gray!50] (4,2) rectangle (5,3);
            \draw[step=1cm] (10,0) grid (15,5);
            \draw[fill=gray!50] (10,3) rectangle (11,4);
            \draw[fill=gray!50] (10,4) rectangle (11,5);
            \draw[fill=gray!50] (11,3) rectangle (12,4);
            \draw[fill=gray!50] (11,4) rectangle (12,5);
            \draw[fill=gray!50] (12,0) rectangle (13,1);
            \draw[fill=gray!50] (12,1) rectangle (13,2);
            \draw[fill=gray!50] (12,3) rectangle (13,4);
            \draw[fill=gray!50] (12,4) rectangle (13,5);
            \draw[fill=gray!50] (10,2) rectangle (11,3);
            \draw[fill=gray!50] (11,2) rectangle (12,3);
            \draw[fill=gray!50] (12,2) rectangle (13,3);
            \draw[fill=gray!50] (13,2) rectangle (14,3);
            \draw[fill=gray!50] (14,2) rectangle (15,3);
        \end{tikzpicture}
      \end{minipage}
      \begin{minipage}{.3\textwidth}
        \caption{Left: symmetric filter. Right: asymmetric filter. Grey indices correspond to indices where the filter is supported.}
        \label{fig:supports5x5}
    \end{minipage}
\end{figure}

As before, we train under augmentation by multiples $\pi/2$ radian rotations of our input images, and to keep the dimensions of input and output the same, we have to modify the convolution to have 2 rows of zero padding. The asymmetric filters are initialized with a $G$-invariant distribution by setting the four indices in the top left corner to 0, in addition to initializing with a standard normal distribution.

The results from this experiment can be seen in Figure \ref{fig:cross_vs_skew5x5} and Figure \ref{fig:cross_vs_skew5x5kl}, where we show the results of the experiment in Section \ref{sec:exp} and the experiment in this section side-by-side. We also give the metrics after epoch $10$ in Table \ref{tab:results5x5}

We can see that the difference in performance, with regards to the metrics of OSP and divergence, between the symmetric and asymmetric models is larger in the case of the $5\times 5$--filters, which is what one would predict if the reason for the size of the discrepancy between symmetric and asymmetric models is the size of the error when approximating $\rho(g)\Pi_\calL A\approx \Pi_\calL\rho(g) A$.
\begin{figure}[!h]
    \centering
    \includegraphics[width=.9\linewidth]{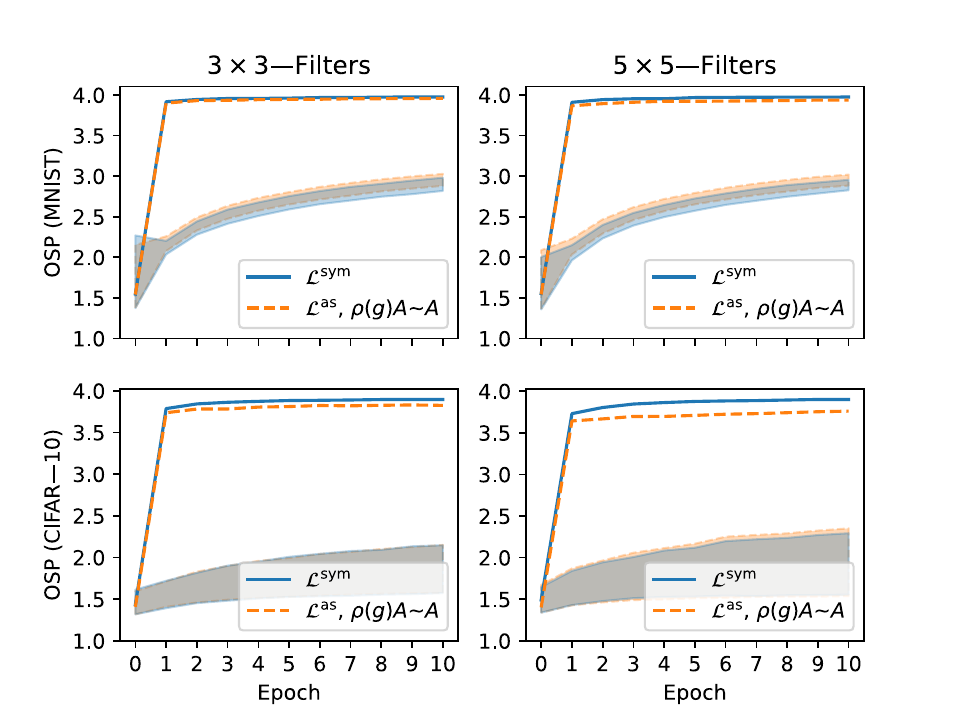}
    \caption{Top: OSP on MNIST test data for ensembles with $1000$ members (higher is better). The middle $95\%$ of individual ensemble members are within the shaded area.\\
    Bottom: OSP on CIFAR--10 test data for ensembles with $1000$ members (higher is better). The middle $95\%$ of individual ensemble members are within the shaded area. Best viewed in color.}
    \label{fig:cross_vs_skew5x5}
\end{figure}
\begin{figure}[!h]
    \centering
    \includegraphics[width=.9\linewidth]{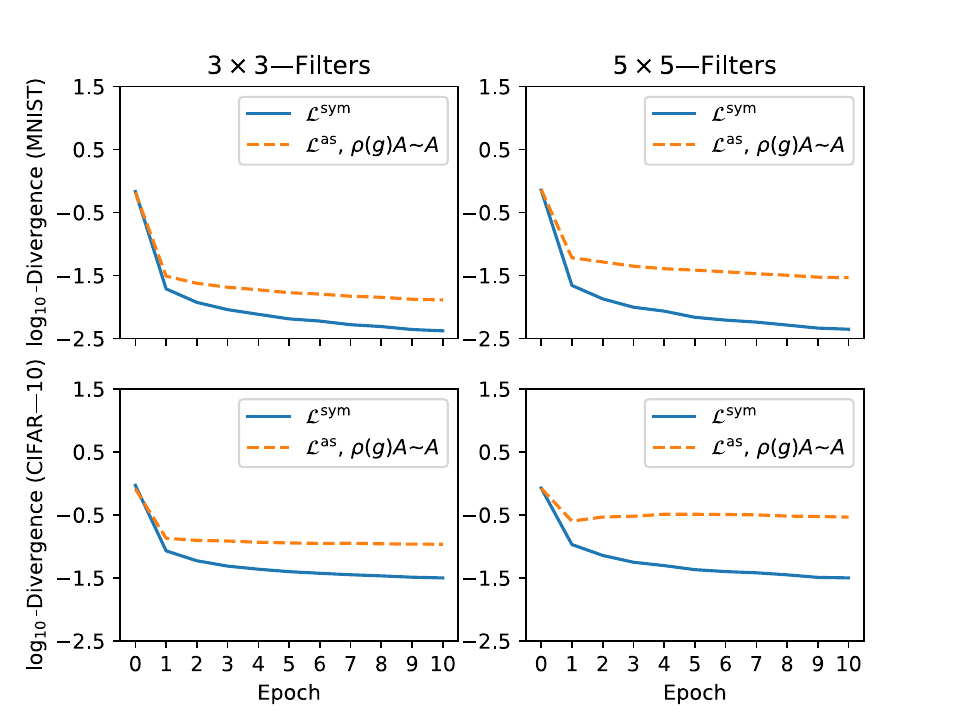}
    \caption{Top: Kullback--Leibler divergence on MNIST test data for ensembles with $1000$ members (lower is better).\\
    Bottom: Kullback--Leibler divergence on CIFAR--10 test data for ensembles with $1000$ members (lower is better). Best viewed in color.}
    \label{fig:cross_vs_skew5x5kl}
\end{figure}
\begin{table}[h]
    \centering
    \begin{tabular}{||c | c | c | c | c||}
        \hline
        \multicolumn{1}{||c}{} & \multicolumn{2}{|c}{MNIST} & \multicolumn{2}{|c||}{CIFAR--10} \\
        \hline
        Model & OSP & $\log D_{\mathrm{KL}}$ & OSP & $\log D_{\mathrm{KL}}$ \\ [0.5ex] 
        \hline\hline
        $\calLsym$ & $\mathbf{3.97}$ & $\mathbf{-2.35}$ & $\mathbf{3.90}$  & $\mathbf{-1.50}$\\ 
        \hline
        $\calLas$, $\rho(g)A\sim A$  & $3.94$  & $-1.54$ & $3.76$  & $-0.54$ \\
        \hline
    \end{tabular}
    \caption{Metrics after the $10$\ts{th} epoch of training for ensembles with $1000$ members using $5\times 5$-filters. Standard deviations are over test data.}
    \label{tab:results5x5}

\end{table}

\end{appendices}
\clearpage
\bibliography{ensembles_equivariant}
\end{document}